\documentclass{SCIS2025}

\usepackage{enumitem}
\usepackage{mathtools}
\newcommand\ci{\perp\!\!\!\perp}
\newcommand\nci{\not\perp\!\!\!\perp}
\usepackage[dvipsnames]{xcolor}
\usepackage[ruled, noend,algo2e]{algorithm2e}

\labelformat{equation}{(#1)}
\newtheorem{innercustomdef}{Definition}
\newenvironment{customdef}[1]{\renewcommand\theinnercustomdef{#1}\innercustomdef}
  {\endinnercustomdef}

\begin{document}
\ArticleType{RESEARCH PAPER}
\Year{2025}
\Month{}
\Vol{}
\No{}
\DOI{}
\ArtNo{}
\ReceiveDate{}
\ReviseDate{}
\AcceptDate{}
\OnlineDate{}

\title{Learning by doing: an online causal reinforcement learning framework with causal-aware policy}{Learning by doing: an online causal reinforcement learning framework with causal-aware policy}

\author[1,2]{Ruichu Cai}{{cairuichu@gmail.com}}
\author[1]{Siyang Huang}{}
\author[1]{Jie Qiao}{}
\author[1]{Wei Chen}{}
\author[3]{Yan Zeng}{}
\author[4]{Keli Zhang}{}
\author[5]{\\Fuchun Sun}{}
\author[6]{Yang Yu}{}
\author[7]{Zhifeng Hao}{}

\AuthorMark{Ruichu Cai}

\AuthorCitation{Ruichu Cai, Siyang Huang, Jie Qiao, et al}

\address[1]{School of Computer Science, Guangdong University of Technology, Guangzhou {\rm 510006}, China}
\address[2]{Pazhou Laboratory (Huangpu), Guangzhou {\rm 510555}, China}
\address[3]{School of Mathematics and Statistics, Beijing Technology and Business University, Beijing {\rm 102401}, China}
\address[4]{Huawei Noah’s Ark Lab, Shenzhen {\rm 518116}, China}
\address[5]{Department of Computer Science and Technology, Tsinghua University, Beijing {\rm 100190}, China}
\address[6]{National Key Laboratory for Novel Software Technology, Nanjing University, Nanjing {\rm 210093}, China}
\address[7]{College of Science, Shantou University, Shantou {\rm 515063}, China}
\abstract{As a key component to intuitive cognition and reasoning solutions in human intelligence, causal knowledge provides great potential for reinforcement learning (RL) agents’ interpretability towards decision-making by helping reduce the searching space. However, there is still a considerable gap in discovering and incorporating causality into RL, which hinders the rapid development of causal RL.
In this paper, we consider explicitly modeling the generation process of states with the causal graphical model, based on which we augment the policy.
We formulate the causal structure updating into the RL interaction process with active intervention learning of the environment.
To optimize the derived objective, we propose a framework with theoretical performance guarantees that alternates between two steps: using interventions for causal structure learning during exploration and using the learned causal structure for policy guidance during exploitation.
Due to the lack of public benchmarks that allow direct intervention in the state space, we design the root cause localization task in our simulated fault alarm environment and then empirically show the effectiveness and robustness of the proposed method against state-of-the-art baselines.
Theoretical analysis shows that our performance improvement attributes to the virtuous cycle of causal-guided policy learning and causal structure learning, which aligns with our experimental results. Codes are available at https://github.com/DMIRLAB-Group/FaultAlarm\_RL.}

\keywords{causal reinforcement learning, reinforcement learning, causality, online reinforcement learning, causal structure learning}

\maketitle

\section{Introduction}

How to decide the next action in repairing the cascading failure under a complex dynamic online system?
Such a question refers to multifarious decision-making problems in which reinforcement learning (RL) has achieved notable success \cite{sutton2018reinforcement,kober2013reinforcement,silver2016mastering,shalev2016safe}.
However, most off-the-shelf RL methods contain a massive decision space and a black-box decision-making policy, thus usually suffering from low sampling efficiency, poor generalization, and lack of interpretability.
As such, current efforts~\cite{sun2021model,zhu2022offline} incorporate domain knowledge and causal structural information into RL to help reduce the searching space as well as improve the interpretability, e.g., a causal structure enables to locate the root cause guiding the policy decision. 
With the causal knowledge, recent RL approaches are mainly categorized as \textit{implicit} and \textit{explicit} modeling-based.

Implicit modeling-based approaches mostly ignore the detailed causal structure and only focus on extracting the task-invariant representations to improve the generalizability in unseen environments \cite{sontakke2021causal,zhang2020learning,tomar2021model,bica2021invariant,sodhani2022improving,wang2021task}. For instance,~\cite{zhang2020learning} proposed a method that extracted the reward-relevant representations while eliminating redundant information. In contrast, explicit modeling-based approaches seek to model the causal structure of the transition of the Markov Decision Process (MDP) \cite{DBLP:conf/nips/DingL0Z22,seitzer2021causal,huang2021adarl,huang2022action,wang2021provably,Luofeng2021,Sergei2020,AmyZhang2019}. For instance,~\cite{huang2022action} proposed a method to learn the causal structure among states and actions to reduce the redundancy in modeling while~\cite{DBLP:conf/nips/DingL0Z22} utilized the causal structure of MDP through a planning-based method. However, these explicit modeling methods either rely on the causal knowledge from domain experts or might suffer from low efficiency in learning policy due to the indirect usage of causal structure in planning and the possible inefficient randomness-driven exploration paradigm.

\begin{figure}[t]
    \centering     
    \includegraphics[width=0.47\textwidth]{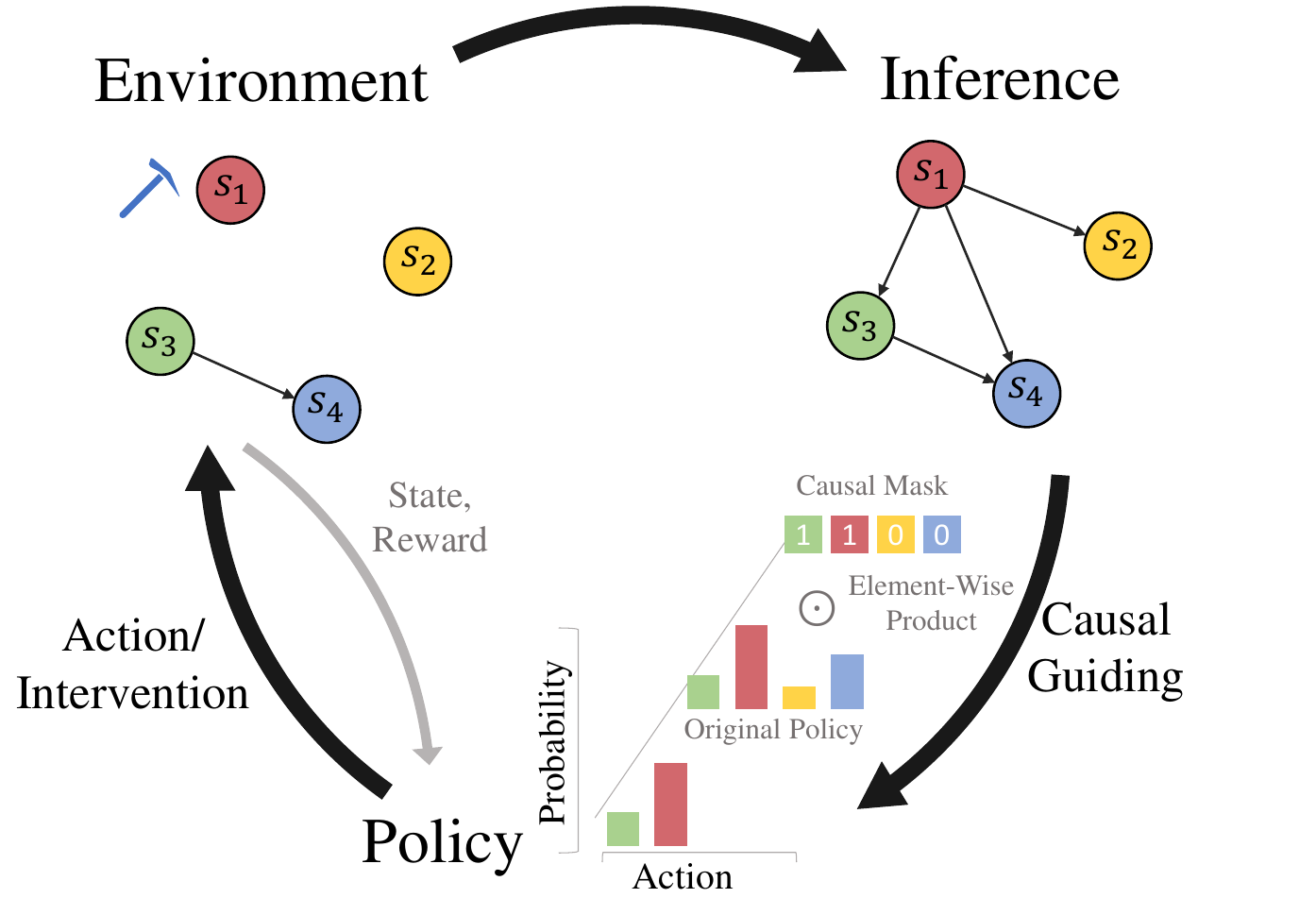}
    \caption{Intervention-Inference-Guidance loop of online causal reinforcement learning.\label{fig:intinfdes}}
\end{figure}

Inspired by the intervention from causality and the decision nature of RL actions in online reinforcement learning: a random action is equivalent to producing an intervention on a certain state such that only its descendants will change while its ancestors will not; a decision could be made according to the causal influence of the action to a certain goal. As such, a causal structure can be learned through interventions by detecting the changing states, which in turn guides a policy with the causal knowledge from the learned causal structure.
Although there has been recent interest in related subjects in causal reinforcement learning, most of them seek to learn a policy either with a fixed prior causal model or a learned but invariant one~\cite{Sergei2020,lee2021causal,huang2022action,DBLP:conf/nips/DingL0Z22}, which does not naturally fit our case when the causal model is dynamically updated iteratively via interventions while learning policy learning (i.e., learning by doing), along with the theoretical identifiability and performance guarantees. 

In this work, as shown in Figure \ref{fig:intinfdes}, we propose an online causal reinforcement learning framework that reframes RL's exploration and exploitation trade-off scheme. 
In exploration, we devise an inference strategy using intervention to efficiently learn the causal structure between states and actions, modeling simultaneously causal dynamics of the environment; while in exploitation, we take the best of the learned structure to develop a causal-knowledge-triggered mask, which leads to a  highly effective causal-aware policy. As such, the causal environment, the causal structure inference strategy, and the causal-aware policy construct a virtuous cycle to the online causal reinforcement learning framework.

In particular, our framework consists of causal structure learning and policy learning.
For causal structure learning, we start by explicitly modeling the environmental causal structure from the observed data as initial knowledge.
Then we formulate the causal structure updating into the RL interaction process with active intervention learning of the environment.
This novel formulation naturally utilizes post-interaction environmental feedback to assess treatment effects after applying the intervention, thus enabling correction and identification of causality.
For policy learning, we propose to construct the causal mask based on the learned causal structure, which helps directly reduce the decision space and thus improves sample efficiency. This leads to an optimization framework that alternates between causal discovery and policy learning to gain generalizability. 
Under some mild conditions, we prove the identifiability of the causal structure and the theoretical performance guarantee of the proposed framework.

To demonstrate the effectiveness of the proposed approach, we established a high-fidelity fault alarm simulation environment in the communication network in the Operations and Maintenance (O\&M) scenario, which requires powerful reasoning capability to learn policies. 
We conduct comprehensive experiments in such an environment, and the experimental results demonstrate that the agent with causal learning capability can learn the optimal policy faster than the state-of-the-art model-free RL algorithms, reduce the exploration risk, and improve the sampling efficiency. 
Additionally, the interaction feedback from the environment can help learn treatment effects and thus update and optimize causal structure more completely. 
Furthermore, our framework with causality can also be unified to different backbones of policy optimization algorithms and be easily applied to other real-world scenarios.

The main contributions are summarized as follows:
\begin{itemize}[leftmargin=0.3in]
\item We propose an online causal reinforcement learning framework, including causal structure and policy learning. It interactively constructs compact and interpretable causal knowledge via intervention (doing), in order to facilitate policy optimization (learning). 
\item We propose a causal structure learning method that automatically updates local causal structures by evaluating the treatment effects of interventions during agent-environment interactions. Based on the learned causal model, we also develop a causal-aware policy optimization method triggered by a causal mask.
\item We derive theoretical guarantees from aspects of both causality and RL: identifiability of the causal structure and performance guarantee of the iterative optimization on the convergence of policy that can be bounded by the causal structure.
\item We experimentally demonstrate that introducing causal structure during policy training can greatly reduce the action space, decrease exploration risk, and accelerate policy convergence.
\end{itemize}

\section{Related work}
\paragraph{Reinforcement learning.} RL solves sequential decision problems by trial and error, aiming to learn an optimal policy to maximize the expected cumulative rewards.
RL algorithms can be conventionally divided into model-free and model-based methods.
The key idea of the model-free method is that agents update the policy based on the experience gained from direct interactions with the environment. 
In practice, model-free methods are subdivided into value-based and policy-based ones. Value-based methods select the policy by estimating the value function, and representative algorithms include deep Q-network (DQN)~\cite{mnih2013playing}, deep deterministic policy gradient (DDPG)~\cite{lillicrap2015continuous}, and dueling double DQN (D3QN)~\cite{wang2016dueling}.
Policy-based methods directly learn the policy function without approximating the value function.
The current mainstream algorithms are proximal policy optimization (PPO)~\cite{schulman2017proximal}, trust region policy optimization (TRPO)~\cite{schulman2015trust}, A2C, A3C~\cite{mnih2016asynchronous} and SAC~\cite{haarnoja2018soft}, etc.
The model-free approach reaches a more accurate solution at the cost of larger trajectory sampling, while the model-based approach achieves better performance with fewer interactions~\cite{kaiser2019model,sutton1991dyna,janner2019trust,garcia1989model,luo2024survey}. 
Despite the better performance of the model-based approach, it is still more difficult to train the environment model, and the model-free approach is more general for real-world applications. In this paper, we apply our approach to the model-free methods.

\paragraph{Causal reinforcement learning.}
Causal RL~\cite{yanzeng2023,de2019causal,sonar2021invariant} is a research direction that combines causal learning with reinforcement learning. 
~\cite{huang2022action} proposed to extract relevant state representations based on the causal structure between partially observable variables to reduce the error of redundant information in decision-making. 
~\cite{sontakke2021causal} and 
~\cite{seitzer2021causal} discovered simple causal influences to improve the efficiency of reinforcement learning.
\cite{lu2020sample} and 
\cite{pitis2020counterfactual} proposed counterfactual-based data augmentation to improve the sample efficiency of RL.
Building dynamic models in model-based RL ~\cite{wang2022causal,sun2021model,zhu2022offline} based on causal graphs has also been widely studied recently. 
~\cite{sun2021model} leveraged the structural causal model as a compact way to encode the changeable modules across domains and applied them to model-based transfer learning. 
~\cite{zhu2022offline} proposed a causal world model for offline reinforcement learning that incorporated causal structure into neural network model learning.
Most of them utilize pre-defined or pre-learned causal graphs as prior knowledge or detect single-step causality to enhance the RL policy learning. However, none of them used the intervention data of the interaction process with the environment to automatically discover or update the complex causal graph. 
Our method introduces a self-renewal interventional mechanism for the causal graph based on causal effects, which ensures the accuracy of causal knowledge and greatly improves the strategy efficiency.
\paragraph{Causal discovery.}
Causal discovery aims to identify the causal relationships between variables. Typical causal discovery methods from observational data are constraint-based methods, score-based methods, and function-based methods. 
Constraint-based methods, such as PC and FCI algorithms \cite{spirtes2000causation}, rely on conditional independence tests to uncover an underlying causal structure. Different from constraint-based methods, Score-based methods use a score to determine the causal direction between variables of interest ~\cite{chickering2002optimal,ramsey2017million,huang2018generalized}. But both constraint-based methods and score-based methods suffer from the Markov Equivalence Class (MEC) problem, i.e., different causal structures imply the same conditional independence tests. By utilizing the data generation process assumptions, like linear non-Gaussian assumption \cite{shimizu2006linear} and the additive noise assumption \cite{hoyer2008nonlinear,peters2014causal,cai2018self}, function-based methods are able to solve the MEC problem and recover the entire causal structure. 

Furthermore, leveraging additional interventional information can provide valuable guidance for the process of causal discovery \cite{brouillard2020differentiable,tigas2022interventions}. An intuitive concept involves observing changes in variables following an intervention on another variable. If intervening in one variable leads to changes in other variables, it suggests a potential causal relationship between the intervened variable and the variables that changed.

\section{Problem formulation}
In this section, we majorly give our model assumption and relevant definitions to formalize the problem.
We concern the RL environment with a Markov Decision Process (MDP) $\left\langle\mathcal{S},\mathcal{A},p,r,\gamma \right\rangle$, where $\mathcal{S}$ denotes the state space, $\mathcal{A}$ denotes the action space, $p(\mathbf{s}'|\mathbf{s},a)$ denotes the dynamic transition from state $\mathbf{s}\in \mathcal{S}$ to the next state $\mathbf{s}'$ when performing action $a\in \mathcal{A}$ in state $\mathbf{s}$, $r$ is a reward function with $r(\mathbf{s},a)$ denoting the reward received by taking action $a$ in state $\mathbf{s}$ and $\gamma\in[0,1]$ is a discount factor. 

To formally investigate the causality in online RL, we make the following factorization state space assumption:
\begin{assumption}[Factorization state space] \label{ass:factorization}
The state variables in the state space $\mathcal{S} =\{s_{1} \times s_{2} \times \dots \times s_{|\mathcal{S}|}\}$ can be decomposed into disjoint components $\{s_{i}\}_{i=1}^{|\mathcal{S}|}$.
\end{assumption}
Assumption~\ref{ass:factorization} implies that the factorization state space has explicit semantics on each state component and thus the causal relationship among states can be well defined. Such an assumption can be satisfied through an abstraction of states which has been extensively studied \cite{tomar2021model,abel2022theory}.

\begin{figure*}[t]
\centering
~~~~~~~~~~~~~~~~~~~~~\subfloat[Full time causal graph in Markov decision process.]{\label{fig:framework_a}
\includegraphics[width=0.35\textwidth]{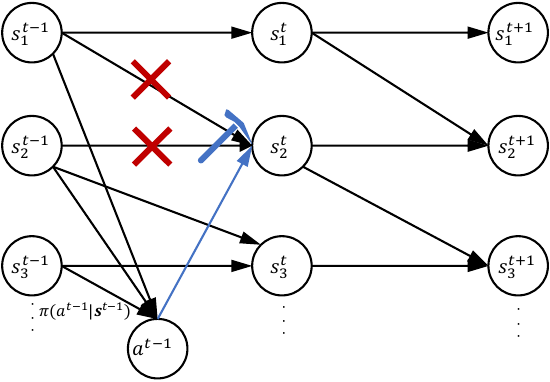}}
~~~~~~~
\begin{minipage}[b]{0.54\textwidth}
\subfloat[Causal graph $\mathcal{G}$ of states in cascade error scenario.]{\label{fig:framework_b}
\includegraphics[width=0.7\textwidth]{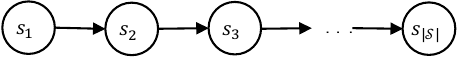}
}

\subfloat[Action on $s_1$ under the causal mask.]{\label{fig:framework_c}
\includegraphics[width=0.31\textwidth]{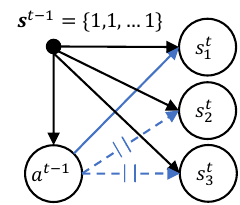}
}
\hspace{0.5cm}
\subfloat[Action on $s_2$ under the causal mask.]{\label{fig:framework_d}
\includegraphics[width=0.31\textwidth]{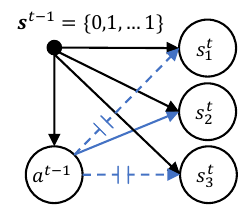}
}
\end{minipage}
    \caption{Illustration of online causal reinforcement learning framework. (a): A full-time causal graph in MDP and the action on the state can be viewed as an intervention. (b) The summary causal graph of (a) where each state would trigger the next state's occurrence, resulting in a cascade error. (c,d): The action from the policy depends on a given situation $S^{t-1}$ as well as the causal mask.}
    \label{fig:framework}
\end{figure*}

\subsection{Causal graphical models and causal reasoning}\label{sec:3.1}
Considering that causality implies the underlying physical mechanism, we can formulate the one-step Markov decision process with the causal graphical model\footnote{Generally, in causality, a directed acyclic graph that represents a causal structure is termed a causal graph~\cite{spirtes2000causation}. Here we generalize each state variable at a timestep $t$ as one variable of interest.}~\cite{peters2017elements} as follows:
\begin{definition}[Causal graph on Markov decision process] \label{def:graph}
Let $\mathcal{G}=(V_\mathcal{S},E)$ denote the causal graph where $V_\mathcal{S}$ is the vertex set defined on the state space, and the edge set $E$ represents the causal relationships among vertex. Given the total time span $[1,2,\dots,T]$, the causal relationship on the one-step transition dynamics can be represented through the factored probability:
\begin{equation}\label{eq:mdp}
p(s_1^t,s_2^t,\dots,s_{|\mathcal{S}|}^t|s_1^{t-1},s_2^{t-1},\dots,s_{|\mathcal{S}|}^{t-1})=\prod_{i=1}^{|\mathcal{S}|}p(s^t_i|\mathbf{s}_{\mathbf{Pa}_i}^{t-1}),
\end{equation}
where $|\mathcal{S}|$ is the support of the state space, $\mathbf{Pa}_i\coloneqq \{s_j|s_j\to s_i \in E\}$ denotes the parent set of $s_i$ according to causal graph $\mathcal{G}$, and $\mathbf{s}_{\mathbf{Pa}_i}^{t-1}$ is the parent states from the last time step.
\end{definition}
To establish a rigorous framework for causal reasoning in MDPs, we introduce the following assumptions, which generalize classical causal assumptions to the temporal domain:
\begin{assumption}[Causal Markov assumption in MDP]
    A causal graph $\mathcal{G}=\{V_{\mathcal{S}},E\}$ and a probability transition distribution $p(\mathbf{s}_i^t|\mathbf{s}^{t-1}_{\mathbf{Pa}_i^{\mathcal{G}}})$ satisfy the Markov condition if and only if for every $s_i^t$ state, $s_i^t$ is independent of $\mathbf{s}_i^{1:T}\setminus \{\mathbf{s}_{\mathbf{Des}^{\mathcal{G}}_i}^{t+1:T}\cup \mathbf{s}_{\mathbf{Pa}^{\mathcal{G}}_i}^{t-1}\}$ given $\mathbf{s}_{\mathbf{Pa}^{\mathcal{G}}_i}^{t-1}$ for all $t$ in MDP, where $\mathbf{s}_i^{1:T}$ denote the set of state variables $i$ from time $1$ to $T$, and $\mathbf{s}_{\mathbf{Des}^{\mathcal{G}}_i}^{t+1:T}$ denotes the descendant of $s_i$ from time $t+1$ to $T$.
\end{assumption}
\begin{assumption}[Causal faithfulness assumption in MDP]
    Let $\mathcal{G}=\{V_{\mathcal{S}},E\}$ be a causal graph and $p(\mathbf{s}_i^t|\mathbf{s}^{t-1}_{\mathbf{Pa}_i^{\mathcal{G}}})$ a transition distribution generated by $\mathcal{G}$. $\langle\mathcal{G}, p\rangle$ satisfies the faithfulness condition if and only if every conditional independence relation true in $p$ is entailed by the causal Markov condition applied to $\mathcal{G}$ at any time in MDP.
\end{assumption}
\begin{assumption}[Causal sufficiency assumption in MDP]
    A set of state variables $V_{\mathcal{S}}$ in $\mathcal{G}$ is causally sufficient if and only if there are no latent confounders of any two observed state variables at any time in MDP.
\end{assumption}
These assumptions are just the generalized version of the original one in the time domain such that causal structure is defined between the last time and the current time using independence. With these assumptions, we can develop the identifiability results for learning causal graph in MDP.

An example of such a causal graph in MDP is given in Figure \ref{fig:framework_a}. In our framework, actions are modeled as interventions, which inherently influence the state. To capture this, we explicitly consider the impact of each action on the state. Without loss of generality, we can model the action on each state as a binary treatment $I_i\in \{0,1\}$ for state $s_i$, where $I_i=0$ indicates the state receives no intervention (natural evolution), and $I_i=1$ indicates the state receives the treatment (treated) under which an intervention is performed. For example, $I_2=1$ at time $t$ in Fig. \ref{fig:framework_a} means that there is an intervention $\operatorname{do}(s_2)$ on $s_2^t$ such that the effect of all parents on $s_2^t$ is removed. Such action modeling is commonly encountered in many scenarios like network operation, robot control, etc. In such a case, we have $p(\operatorname{do}(s_2^t)|\mathbf{s}_{\mathbf{Pa}_2}^{t-1})=p(\operatorname{do}(s_2^t))$ \cite{pearl2009causality}. 
As such, the policy serves as the treatment assignment for each state, and the action space is structured such that each dimension corresponds to a binary intervention on a specific state variable (i.e., $I_i$ for $s_i$). This design ensures that the action space spans the same dimensions as the state space: every state variable has an associated intervention ``lever'' in the action space so that we can intervene the state and measure the effect of certain outcomes. This allows us to learn the causal influence within each state, which will further improve policy learning by selecting the most influenced action to the goal. While we assume full intervention capability across all state dimensions for simplicity, this framework readily extends to scenarios where certain states remain non-intervenable by omitting their corresponding action dimensions. Based on Definition~\ref{def:graph}, we can define the average treatment effect among states.
\begin{definition}[Average Treatment Effect (ATE) on states]\label{eq:causal_effect}
    Let $s_i$ and $s_j$ denote two different state variables. Then the treatment effect of $s_i$ on $s_j$ is, 
    \begin{equation}
        \begin{aligned}
            \mathcal{C}_{s_i\to s_j}=\mathbb{E}[s_j(I_i=1)-s_j(I_i=0)],
        \end{aligned}
    \end{equation}
    where $s_j(I_i=1)$ denotes the potential outcome of $s_j$ if $s_i$ were treated (intervened), $s_j(I_i=0)$ denotes the potential outcome if $s_i$ were not treated \cite{rosenbaum1983central}.
\end{definition}
Intuitively, the potential outcome depicts the outcome of the state in performing different treatments and the ATE evaluates the treatment effect on the outcome. That is, ATE answers the question that when an agent performs an action $\operatorname{do}(s_i)$, how is the average cause of an outcome of $s_j$ \cite{pearl2009causality}? Such a question suggests that an action applied to a state will solely influence its descendants and not its ancestors. This aspect is crucial for causal discovery, as it reveals the causal order among the states. Moreover, the treatment is not necessarily binary since our goal is to infer the causal order by the property of intervention in action, i.e., an intervention on the cause will influence its effect, which is also held in multi-treatment \cite{lopez2017estimation} or the continue-treatment \cite{callaway2024difference}. One can simply modify the corresponding ATE to adapt to the general treatment. For simplicity, we assume binary treatment in this work. To further accomplish the causal discovery, we assume that the states satisfy the causal sufficiency assumption \cite{peters2017elements}, i.e., there are no hidden confounders and all variables are observable.

\section{Framework}
In this section, with proper definitions and assumptions, we first propose a general online causal reinforcement learning framework, which consists of two phases: policy learning and causal structure learning.
Then, we describe these two phases in detail and provide a performance guarantee for them.
The overall flow of our framework is eventually summarized in Algorithm \ref{algo}.

\subsection{Causal-aware policy learning}
The general objective of RL is to maximize the expected cumulative reward by learning an optimal policy $\max_{\pi}\mathbb {E}\left[\sum_{t=0}^{T}\gamma^t r(\mathbf{s}^t,a^t)\right]$.
Inspired by viewing the action as the intervention on state variables, we use the fact that the causal structure $\mathcal{G}$ among state variables is effective in improving the policy decision space, proposing the causal-aware policy $\pi_{\mathcal{G}}(\cdot|\mathbf{s})$ with the following objective function for optimization:
\begin{equation}
    \max_{\pi_{\mathcal{G}}}\mathbb {E}\left[\sum_{t=0}^{T}\gamma^t r(\mathbf{s}^t,a^t)\right].
\end{equation}

Let us consider a simple case where we have already obtained a causal graph $\mathcal{G}$ of the state-action space.
We now define a causal policy and associate it with the state-space causal structure $\mathcal{G}$:
\begin{definition}[Causal policy]\label{def:causal_policy}
    Given a causal graph $\mathcal{G}$ on the state space, we define the causal policy $\pi_{\mathcal{G}}(\cdot|\mathbf{s})$ under the causal graph $\mathcal{G}$ as follows:
    \begin{equation}
    \pi_{\mathcal{G}}(\cdot|\mathbf{s})=M_{\mathbf{s}}(\mathcal{G})\circ\pi(\cdot|\mathbf{s}),
    \end{equation}
    where $M_{\mathbf{s}}(\mathcal{G})$ is the causal mask vector at state $\mathbf{s}$ w.r.t. $\mathcal{G}$, $\pi(\cdot|\mathbf{s})$ is the action probability distribution, and $\pi_{\mathcal{G}}(\cdot|\mathbf{s})$ is the distribution of causal policy where each action is masked according to $M_{\mathbf{s}}(\mathcal{G})$.
\end{definition}
 The causal mask $M_{\mathbf{s}}(\mathcal{G})=\{m_{\mathbf{s},a}^{\mathcal{G}}\}_{a=1}^{|\mathcal{A}|}$ is induced by the causal structure and the current state, aiming to pick out causes of the state and refine the searching space of policy. In other words, it ensures that all irrelevant actions can be masked out.
For example, in a cascade error scenario of communication in Fig. \ref{fig:framework_b}, where each state  (e.g., system fault alarm) would trigger the next state's occurrence, resulting in cascade and catastrophic errors in communication networks, 
the goal here is to learn a policy that can quickly eliminate system fault alarms. The most effective and reasonable solution is to intervene on the root cause of the state, to prevent possible cascade errors.
In Fig. \ref{fig:framework_b}, we should intervene on $s_2$ since $s_1$ is not an error and $s_2$ is the root cause of the system on its current state. 

For more general cases, based on the causal structure of errors, we can obtain the $TopK$ causal order representing $K$ possible root-cause errors and construct the causal mask vector to refine the decision space to a subset of potential root-cause errors. This is, the $i$-th element in $M_{\mathbf{s}}(\mathcal{G})$ is not masked ($m_{\mathbf{s},i}^{\mathcal{G}}=0$) only if $s_{i} \in TopK_{\tilde{\mathcal{G}}}$ where $TopK_{\tilde{\mathcal{G}}}$ is the $TopK$ causal order of $\tilde{\mathcal{G}}$, and $\tilde{\mathcal{G}} \coloneqq \mathcal{G} \setminus \left\{s_{i} |s_{i}^{t} =0\right\}$, $K$ denotes the number of candidate causal actions.
It is worth mentioning that different tasks correspond to different causal masks, but the essential role of the causal mask is to use causal knowledge to retain task-related actions and remove task-irrelevant actions, thus helping the policy to reduce unnecessary sampling. For example, for some goal $Y$, the causal mask can be set to $m_{\mathbf{s},i}^{\mathcal{G}} \propto |\mathcal{C}_{i\to y}|$ which is proportional to the causal effect where $\mathcal{C}_{i\to y}=\mathbb{E}[Y(I_i=1)-Y(I_i=0)]$ so that the causal mask can be task-specific for different goal.
Note that some relevant causal imitation learning algorithms exist that utilize similar mask strategies~\cite {de2019causal,samsami2021causal}. However, they focus on imitation learning settings other than reinforcement learning. And they use the causal structure accurately while we take the best of causal order information, allowing the presence of transitory incomplete causal structures in iterations and improving computational efficiency.

In practice, we use an actor-critic algorithm PPO~\cite{schulman2017proximal} as the original policy, which selects the best action via maximizing the Q value function $Q(\mathbf{s}^t,a^t)$.
Notice that our method is general enough to be integrated with any other RL algorithms.

\begin{algorithm}[t]
\small
\caption{Online causal reinforcement learning training process}
\label{algo}
\KwIn{Policy network $\theta$; Replay buffer $\mathcal{B}$; Causal structure $\mathcal{G}$}
\While{$\theta$ not converged}
{
\tcp{Causal-aware policy learning} 
\While{$t < T$}
{
    
    $a^t \leftarrow$ Causal policy $\pi_{\mathcal{G}}(\cdot |\mathbf{s}^t)$ with causal mask $M_{\mathbf{s}^t}(\mathcal{G})$
    
    $\mathbf{s}^{t+1}$, $r^t \leftarrow$ Env($\mathbf{s}^t,a^t$) 
    
    $\mathcal{B} \leftarrow \mathcal{B}\cup\{a^t, \mathbf{s}^t, r^t, \mathbf{s}^{t+1} \} $
}
\tcp{Causal structure learning} 
\For{$i \leq |\mathcal{S}|$}
{
    \For{$j \leq |\mathcal{S}|$}
    {
        Estimate $\hat{\mathcal{C}}^{Att}_{s_{i}\to s_{j}}$ from $\mathcal{B}$ 

        Infer the causal relation between $s_{i}, s_{j}$ based on $\hat{\mathcal{C}}^{Att}_{s_{i}\to s_{j}}$ (Theorem \ref{thm:causal order}).
    }
}
Prune redundant edges of $\mathcal{G}$

Update $\theta$ with $\mathcal{B}$
}
\end{algorithm}
\subsection{Causal structure learning}

In this phase, we relax the assumption of giving $\mathcal{G}$ as a prior and aim to learn the causal structure through the online RL interaction process.
As discussed before, an action is to impose a treatment and perform an intervention on the state affecting only its descendants while not its ancestors. As such, we develop a two-stage approach for learning causal structure with orientation and pruning stages. 

In the orientation stage, we aim to estimate the treatment effect for each pair to identify the causal order of each state.
However, due to the counterfactual characteristics in the potential outcome~\cite{pearl2009causality}, i.e., we can not observe both control and treatment happen at the same time, and thus a proper approximation must be developed. In this work, instead of estimating ATE, we propose to estimate the \textit{Average Treatment effect for the Treated sample} (ATT) \cite{Susan_ATT}:
\begin{equation}
\!\!\!\hat{\mathcal{C}}^{Att}_{s_i\to s_j}=\frac{1}{n}\!\!\!\sum_{\{k:I_i=1\}}\!\!\![s_j^{(k)}(I_i=1)-\hat{s}_j^{(k)}(I_i=0)],
\end{equation}
where $n$ denotes the number of treated samples when $I_i=1$, $s_j^{(k)}(I_i=1)$ is the $k$-th observed sample, and $\hat{s}_j^{(k)}(I_i=0)$ is an estimation that can be estimated from the transition in Eq. \ref{eq:mdp}.
\begin{theorem}\label{thm:causal order}
    Given a causal graph $\mathcal{G}=(V_\mathcal{S},E)$, for each pair of states $s_i,s_j$ with $i \ne j$, $s_i$ is the ancestor of $s_j$, i.e., $s_i$ has a direct path to $s_j$ if and only if $|\mathcal{C}^{Att}_{s_i\to s_j}|>0$.
\end{theorem}
Please see \ref{sec_appdendix:theory} for detailed proofs of all theorems and lemma. Theorem \ref{thm:causal order} ensures that ATT can be used to identify the causal order. However, redundant edges might still exist even when accounting for the causal order.  To address this, we introduce a pruning stage and formulate a pruning method using a score-based approach to refine the causal discovery results. Specifically, the aim of causal structure learning can be formalized as maximizing the score of log-likelihood with an $\ell_0$-norm penalty:
\begin{equation}\label{eq:likelihood}\max_{G}\sum_{t=1}^{T}\sum_{i=1}^{|\mathcal{S}|}\log p(s^t_i|\mathbf{s}_{\mathbf{Pa}_i}^{t-1}))-\alpha \|\mathbf{G}\|_0,
\end{equation}
where $\mathbf{G}$ is the adjacency matrix of the causal graph \cite{chickering2002optimal}. Note that such $\ell_0$-norm can be relaxed to a quadratic penalty practically for optimization \cite{zheng2018dags} but we stick to the $\ell_0$-norm here for the theoretical plausibility.
Then by utilizing the score in Eq. \ref{eq:likelihood}, we can prune the redundant edges by checking whether the removed edge can increase the score above. We continue the optimization until no edge can be removed. By combining the orientation and the pruning stage, the causal structure is identifiable, which is illustrated theoretically in Theorem~\ref{thm:identifiability}.
\begin{theorem}[Identifiability]\label{thm:identifiability}
    Under the causal faithfulness and causal sufficiency assumptions, given the correct causal order and large enough data, the causal structure among states is identifiable from observational data.
\end{theorem}

\subsection{Performance guarantees}
To analyze the performance of the optimization of the causal policy, we first list the important Lemma \ref{lemma_1} where the differences between two different causal policies are highly correlated with their causal graphs, and then show that policy learning can be well supported by the causal learning.

\begin{lemma}\label{lemma_1}
Let $\pi_{\mathcal{G}^{*}}(\cdot |\mathbf{s})$ be the policy under the true causal graph $\mathcal{G}^{*}=\left(V_\mathcal{S},E^{*}\right)$. For any causal graph $\mathcal{G}=(V_\mathcal{S},E)$, when the defined causal policy $\pi_{\mathcal{G}}(\cdot |\mathbf{s})$ converges, the following inequality holds: 
\begin{equation}
\begin{split}
D_{TV}(\pi_{\mathcal{G}^{*}},\pi_{\mathcal{G}}) \leq & \dfrac{1}{2} (\|M_{\mathbf{s}}(\mathcal{G})-M_{\mathbf{s}}(\mathcal{G}^{*})\|_{1} 
\\& + \Vert\mathbf{1}_{\{a:m_{\mathbf{s},a}^{\mathcal{G}^{*}}=1\land m_{\mathbf{s},a}^{\mathcal{G}}=1 \} } \Vert_{1},
\end{split}
\end{equation}
where $\displaystyle \| M_{\mathbf{s}} (\mathcal{G} )-M_{\mathbf{s}}\left(\mathcal{G}^{*}\right) \| _{1}$ is the $\displaystyle \ell _{1}$-norm of the masks measuring the differences of two policies, $\displaystyle \mathbf{1}$ is an indicator function and $\displaystyle \| \mathbf{1}_{\left\{a:m_{\mathbf{s},a}^{\mathcal{G}^{*}} =1\land m_{\mathbf{s},a}^{\mathcal{G}} =1\right\}} \| _{1}$ measures the number of actions that are not masked on both policies.
\end{lemma}
Lemma \ref{lemma_1} shows that the total variation distance between two polices $\pi_{\mathcal{G}^{*}}(\cdot |\mathbf{s})$ and $\pi_{\mathcal{G}}(\cdot |\mathbf{s})$, is upper bounded by two terms that depend on the divergence between the estimated causal structure (causal masks) and the true one.
It bridges the gap between causality and reinforcement learning, which also verifies that causal knowledge matters in policy optimization. 
In turn, this lemma facilitates the improvement of the value function's performance, as shown in Theorem~\ref{thm:value_bound}. 

\begin{theorem}\label{thm:value_bound}
Given a causal policy $\pi_{\mathcal{G}^{*}}(\cdot|\mathbf{s})$ under the true causal graph $\mathcal{G}^{*}=\left(V_\mathcal{S},E^{*}\right)$ and a policy $\pi_{\mathcal{G}}(\cdot |\mathbf{s})$ under the causal graph $\mathcal{G} =(V_\mathcal{S},E)$, recalling $R_{\max}$ is the upper bound of the reward function, we have the performance difference of $\pi_{\mathcal{G}^{*}}(\cdot|\mathbf{s})$ and $\pi_{\mathcal{G}}(\cdot|\mathbf{s})$ be bounded as below,
\begin{equation}
\begin{split}
\begin{aligned}
V_{\pi_{\mathcal{G}^{*}}} -V_{\pi_{\mathcal{G}}}  \leq &\frac{R_{\max}}{(1-\gamma )^{2}}(\| M_{\mathbf{s}}(\mathcal{G})-M_{\mathbf{s}}(\mathcal{G}^{*}) \|_{1} \\& + \Vert \mathbf{1}_{\{a:m_{\mathbf{s},a}^{\mathcal{G}^{*}}=1\land m_{\mathbf{s},a}^{\mathcal{G}}=1 \}}\Vert_{1}).
\end{aligned}
\end{split}
\end{equation}
\end{theorem}

An intuition of performance guarantees is that policy exploration helps to learn better causal structures through intervention, while better causal structures indicate better policy improvements. The detailed proofs of the above lemma and theorems are in Appendix.

\section{Experiments}
In this section, we first discuss the basic setting of our designed environment as well as the baselines used in the experiments.
Then, to evaluate the proposed approach, we conducted comparative experiments on the environment and provide the numerical results and detailed analysis.
\begin{figure}[t]
	\centering
    \subfloat[] {
        \label{fig:policy_reward}
        \centering
        \includegraphics[width=0.3\linewidth]{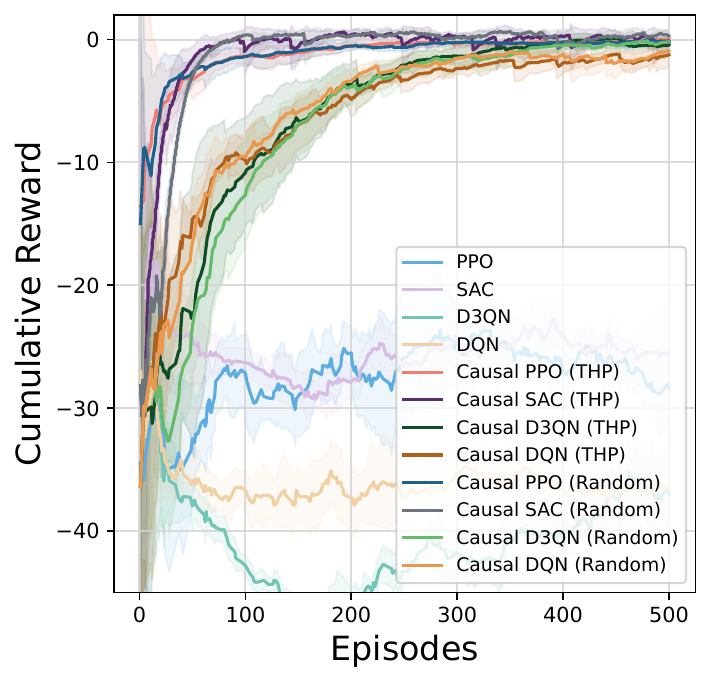}
    }
	\centering
    \subfloat[] {
        \label{fig:policy_steps}
        \centering
        \includegraphics[width=0.3\linewidth]{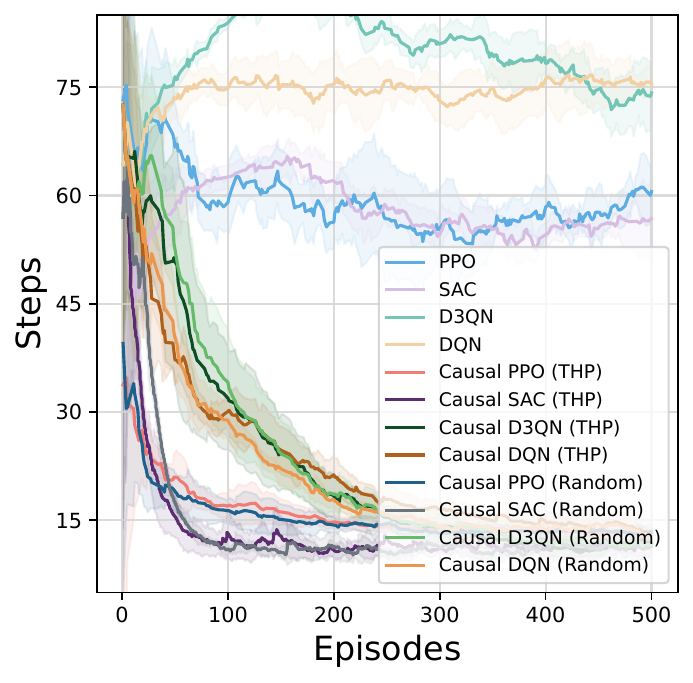}
    }
    
	\centering
    \subfloat[] {
        \label{fig:policy_alarms}
        \centering
        \includegraphics[width=0.3\linewidth]{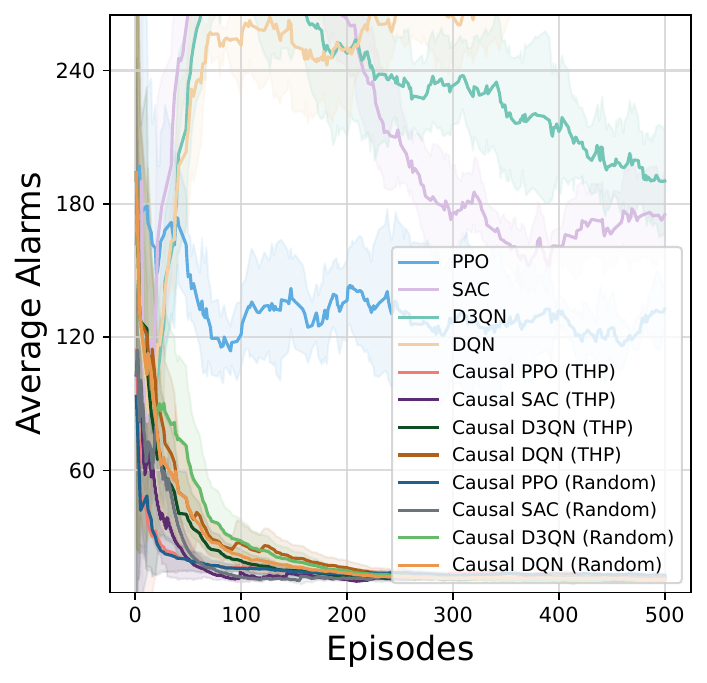}
    }
    \centering
    \subfloat[] {
        \label{fig:f1}
        \centering
        \includegraphics[width=0.3\linewidth]{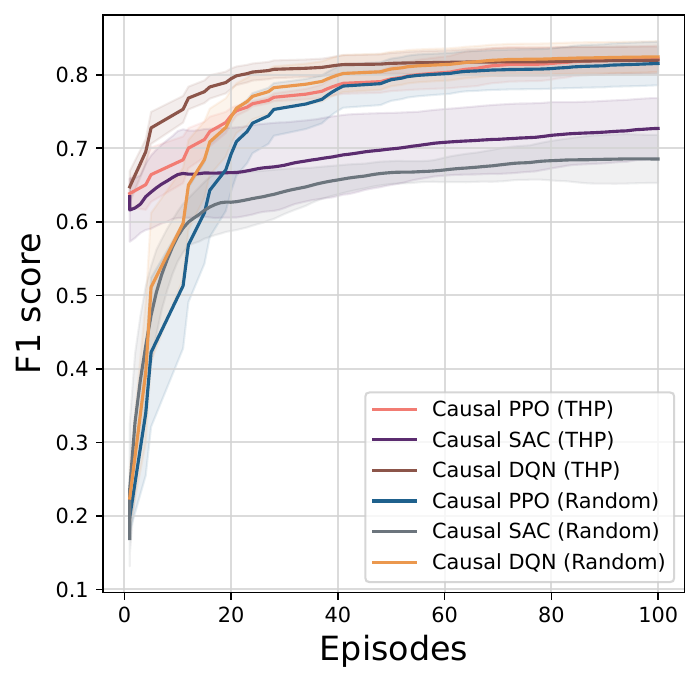}
    }
    \centering
    \caption{(a) Cumulative rewards of Causal PPO, Causal SAC, Causal D3QN, Causal DQN with THP-initialized structures and random-initialized structures, respectively, and baselines; (b) Intervention steps of our proposed approach compared to the baselines; (c) Average number of alarms per episode for our methods compared to the baselines; (d) The F1 score of causal structure learning from different methods.} 
    \label{fig:exp_policy_result}
\end{figure}
\begin{table*}
\liuhao
    \centering
    \caption{Results of causal structure learning}
    \begin{tabular}{l|c|c|c|c|c}
        \toprule
         Methods & F1 score & Precision & Recall & Accuracy & SHD\\
        \midrule
        THP & 0.638 $\pm$ 0.017 & 0.775 $\pm$ 0.020 & 0.543 $\pm$ 0.015 & 0.824 $\pm$ 0.007 & 57.00 $\pm$ 2.121
        \\
        Causal PPO (THP) & \textbf{0.861} $\pm$ 0.018 & 0.865 $\pm$ 0.007 & \textbf{0.856} $\pm$ 0.029 & \textbf{0.921 $\pm$ 0.009 }& \textbf{26.00 $\pm$ 2.915}
        \\
        Causal SAC (THP) & 0.858 $\pm$ 0.013 & \textbf{0.871 $\pm$ 0.007} & 0.846 $\pm$ 0.024 & 0.919 $\pm$ 0.007 & 26.25 $\pm$ 2.165
        \\
        Causal D3QN (THP) & 0.836 $\pm$ 0.015 & 0.849 $\pm$ 0.021 & 0.823 $\pm$ 0.014 & 0.904 $\pm$ 0.009 & 31.00 $\pm$ 3.000
        \\
        Causal DQN (THP) & 0.832 $\pm$ 0.020 & 0.848 $\pm$ 0.020 & 0.817 $\pm$ 0.025 & 0.904 $\pm$ 0.013 & 31.00 $\pm$ 4.062
        \\
        \midrule
        Random Initiation & 0.188 $\pm$ 0.013 & 0.130 $\pm$ 0.009 & 0.130 $\pm$ 0.009 & 0.669 $\pm$ 0.017 & 107.5 $\pm$ 5.362
        \\
        Causal PPO (Random) & \textbf{0.840} $\pm$ 0.019 & 0.847 $\pm$ 0.015 & \textbf{0.834} $\pm$ 0.025 & \textbf{0.909 $\pm$ 0.011} & 29.50 $\pm$ 3.640
        \\
        Causal SAC (Random) & 0.837 $\pm$ 0.019 & \textbf{0.864 $\pm$ 0.015} & 0.811 $\pm$ 0.022 & 0.908 $\pm$ 0.010 & \textbf{29.75 $\pm$ 3.269}
        \\
        Causal D3QN (Random) & 0.839 $\pm$ 0.016 & 0.847 $\pm$ 0.022 & 0.832 $\pm$ 0.017 & 0.907 $\pm$ 0.011 & 30.25 $\pm$ 3.491
        \\
        Causal DQN (Random) & 0.830 $\pm$ 0.019 & 0.849 $\pm$ 0.025 & 0.813 $\pm$ 0.020 & 0.904 $\pm$ 0.013 & 31.25 $\pm$ 4.085
        \\
        \bottomrule
    \end{tabular}
    \label{tab_causal}
\end{table*}

\subsection{Environment design}
Since most commonly used RL benchmarks do not explicitly allow causal reasoning, we constructed \textit{FaultAlarmRL}, a simulated fault alarm environment based on the real alarm data in the real-world application of wireless communication networks \cite{cai2022thps}.

FaultAlarmRL environment is designed to mimic the operation process in a large communication network within a Markov Decision Process (MDP) framework. In the Operations and Maintenance (O\&M) process of such networks, efficiently and accurately locating the root cause of alarms within a given time period is crucial. Timely fault elimination improves O\&M efficiency and ensures communication quality. In real wireless networks, the alarm event sequences of different nodes influence each other through the node topology, and the causal mechanisms between different types of alarm events are also affected by the underlying topology.

The simulation environment contains 50 device nodes and 18 alarm types, with the true causal relationships between alarm types and the meaning of each alarm type shown in Table~\ref{tab:ground_truth}. Alarm events are generated by root cause events based on the alarm causal graph and device topology graph propagation. There also exist spontaneous noise alarms in the environment. To mimic the operation in a large communication network, we designed an MDP transition environment modified from the topological Hawkes process. For example, the number of alarm events that occur in $X_{t+1}$ is determined by the number of alarms in the previous time interval $X_t$ without decay. This means that alarms persist until they are "fixed" and this type of transition constructs an MDP environment where the alarm propagation process can be expressed as:
$$
\begin{aligned}
p(s_{t+1}|s_{t},a_{t};G_V,G_N)
&= P(\mathbf{X}_{t+1} | \mathbf{X}_{t};G_V,G_N) \\
&= \prod_{n\in N,v\in V} P(X_{n,v,t+1} | X_{n,PA_{v},t}) \\
&= \prod_{n\in N,v\in V} \operatorname{Pois}(X_{n,v,t+1}; \lambda_{v}(n,t+1)),
\end{aligned}
$$
where $X_{n,v,t+1}$ is the count of occurrence events of event type $v$ at node $n$ in the time interval $[t+1-\Delta t, t+1]$, $\operatorname{Pois}$ is the Poisson distribution, and $\lambda_v(n,t)$ is the Hawkes process intensity function. Specifically, $\lambda_v(n,t)$ is defined as:
$$
\lambda _{v}(n,t) = \mu _{v} + \sum _{v'\in PA_{v}} \sum _{n'\in N} \sum _{k=0}^{K} \alpha _{v',v,k} \hat{A}_{n',n}^{K} \kappa X_{n',v',t-1},
$$
where $X_{n,v,t-1}$ is the count of occurrence alarms of type $v$ at node $n$ in the time interval $[t-1-\Delta t, t-1]$, $\kappa$ is the exponential kernel function, $k$ is the maximum hop, $\alpha _{v',v,k}$ is the propagation intensity function of the alarm, $\hat{A} \coloneqq D^{-1/2}AD^{-1/2}$ is the normalized adjacency matrix of the topological graph, $A$ is the adjacency matrix, $D$ is the diagonal degree matrix, $\hat{A}_{n',n}^K$ denotes the $n',n$-th entries of the $K$-hop topological graph, and $\mu _{v}$ is the spontaneous intensity function of the alarm $v$.

The state in FaultAlarmRL is the current observed alarm information, which includes the time of the fault alarm, the fault alarm device, and the fault alarm type. The state space has $50 \times 18 \times 2 = 1800$ dimensions. The action space contains 900 discrete actions, each of which represents a specific alarm type on a specific device. We define the reward function as:
$$
r = \frac{N_t - N_{t+1}}{N_t} - \frac{t}{\text{step}_{\text{max}}},
$$
where $N_t$ represents the number of alarms at time $t$, and $\text{step}_{\text{max}}$ is the maximum number of steps in an episode, which is set to 100. Please see Section~\ref{sec:hyper-parameters} for further details on the hyper-parameters of the environment. Additionally, we further evaluate our method in \textit{cart-pole} environment from the OpenAI Gym toolkit (see Section~\ref{sec:carpole}).

\subsection{Experimental setups}
We evaluate the performance of our methods in terms of both causal structure learning and policy learning.
We first sampled 2000 alarm observations from the environment for the pre-causal structure learning.
We learn the initial causal structure leveraging the causal discovery method topological Hawkes process (THP)~\cite{cai2021thp} that considers the topological information behind the event sequence.
In policy learning, we take the SOTA model-free algorithms PPO~\cite{schulman2017proximal}, SAC~\cite{haarnoja2018soft}, D3QN~\cite{wang2016dueling}, and DQN~\cite{mnih2013playing} which are suitable for discrete cases as the baselines, and call the algorithms after applying our method Causal PPO, Causal SAC, Causal D3QN, Causal DQN.
For a fair comparison, we use the same network structure, optimizer, learning rate, and batch size when comparing the native methods with our causal methods.
We measure the performance of policy learning in terms of cumulative reward, number of interactions, and average number of alerts per episode. In causal structure learning, Recall, Precision, F1, Accuracy and SHD are used as the evaluation metrics.
All results were averaged across four random seeds, with standard deviations shown in the shaded area.

\begin{figure}[t]
	\centering
    \subfloat[] {
        \centering
        \includegraphics[width=0.3\linewidth]{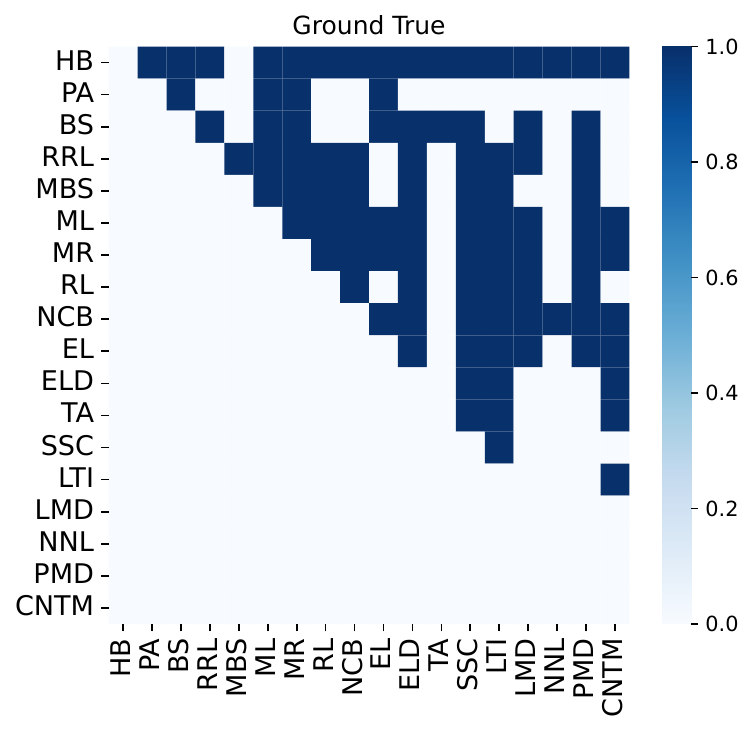}
    }
	\centering
    \subfloat[] {
        \centering
        \includegraphics[width=0.3\linewidth]{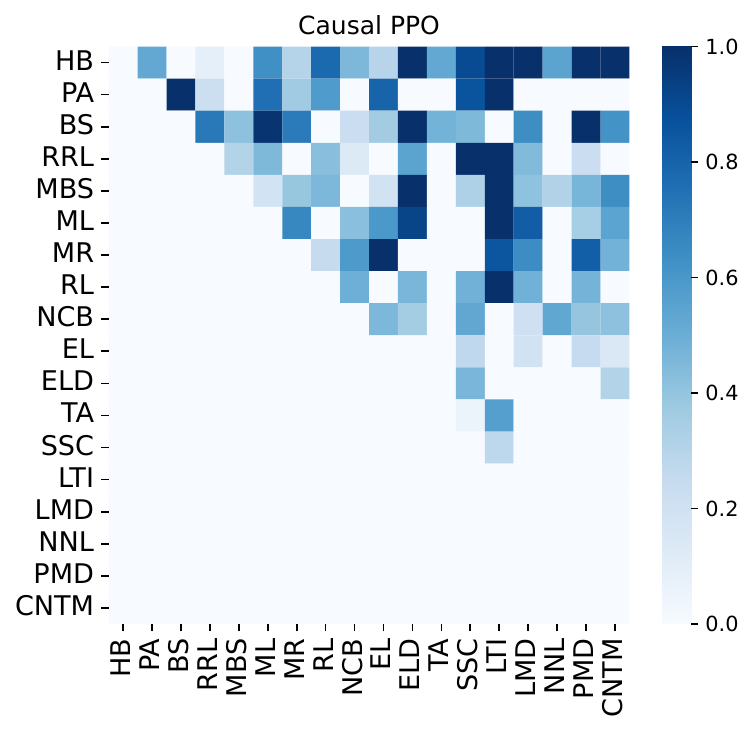}
    }
    \centering
    \subfloat[] {
        \centering
        \includegraphics[width=0.3\linewidth]{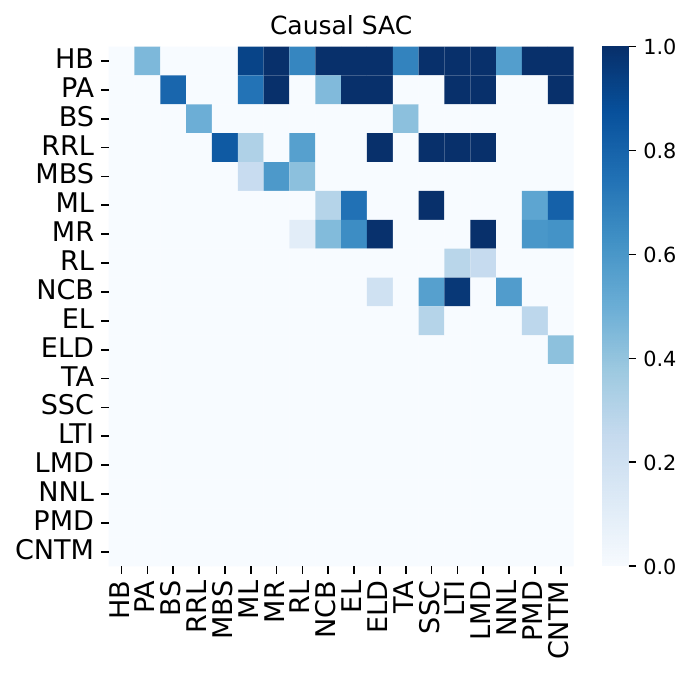}
    }
    
    \centering
    \subfloat[] {
        \centering
        \includegraphics[width=0.3\linewidth]{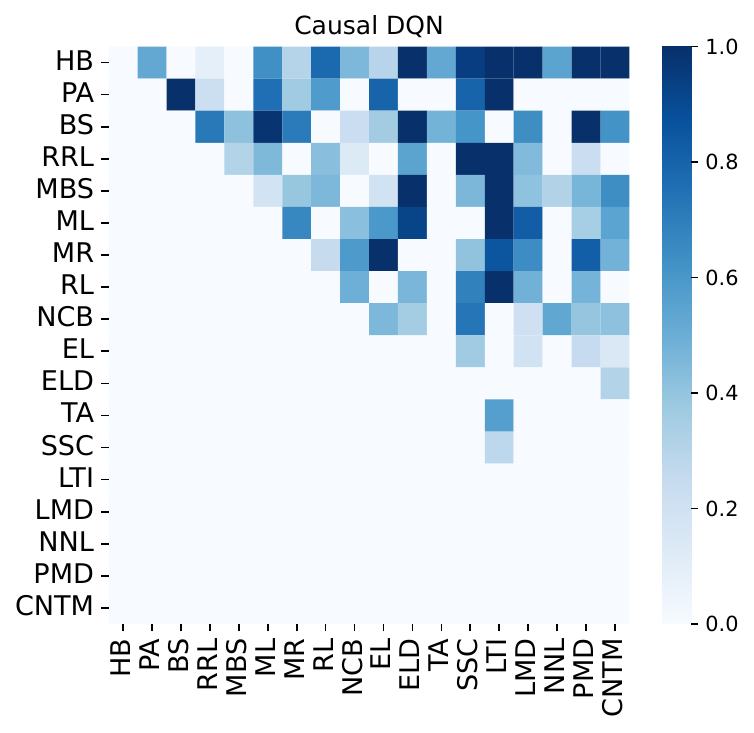}
    }
    \centering
    \subfloat[] {
        \centering
        \includegraphics[width=0.3\linewidth]{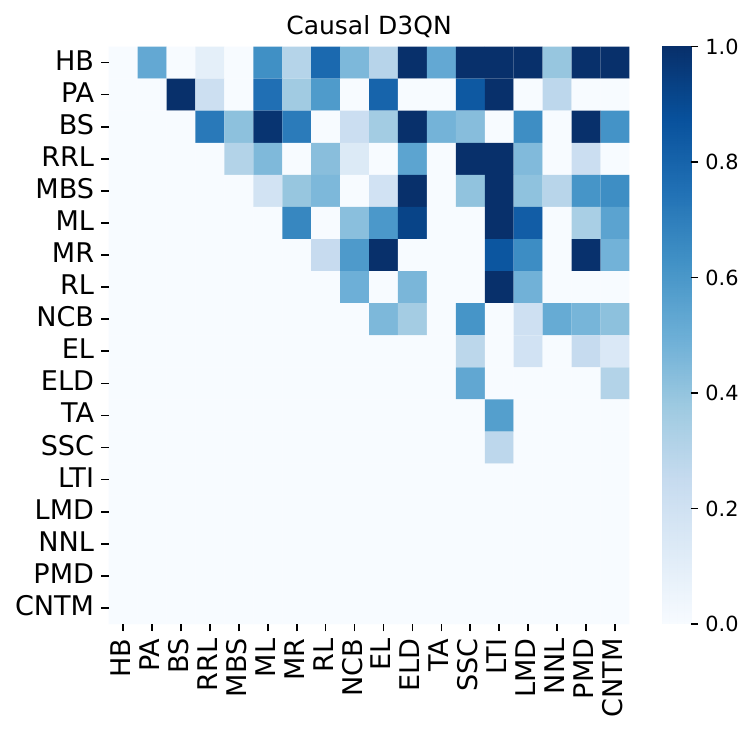}
    }
    \centering
    \caption{(a) Ground truth; (b-e) Discovered causal graphs by Causal PPO, Causal SAC, Causal DQN, Causal D3QN with THP-initialized causal structure.} 
    \label{fig:exp_causal_result}
\end{figure}
\subsection{Analysis of policy learning}
To evaluate the performance of our methods, the cumulative reward, the number of interventions, and the average number of alarms are used as evaluation metrics.
As shown in Figure~\ref{fig:policy_reward}, our methods significantly outperform the native algorithms after introducing our framework. It can be found that our algorithms only need to learn fewer rounds to reach higher cumulative rewards, which proves that the learned causal structure indeed helps to narrow the action space, and greatly speed up the convergence of the policy.

We also show the results of different algorithms on the number of intervention steps in Figure~\ref{fig:policy_steps}. Impressively, our method requires fewer interventions to eliminate all the environmental alarms and does not require excessive exploration in the training process compared with the baselines. 
This is very important in real-world O\&M processes, because too many explorations may pose a huge risk.
The above result also reflects that policies with causal structure learning capabilities have a more efficient and effective training process and sampling efficiency.

From Figure~\ref{fig:policy_alarms}, we can also see that our method has much smaller average number of alarms compared with the baselines. This indicates that our methods can detect root cause alarms in time, and thus avoid the cascade alarms generated from the environment.
It is worth noting that the huge performance difference between our methods and baselines shows that the learned causal mechanisms of the environment play a pivotal role in RL.

\begin{figure}[t]
	\centering
    \subfloat[] {
    \label{fig:policy_reward_ablition}
        \centering
        \includegraphics[width=0.3\linewidth]{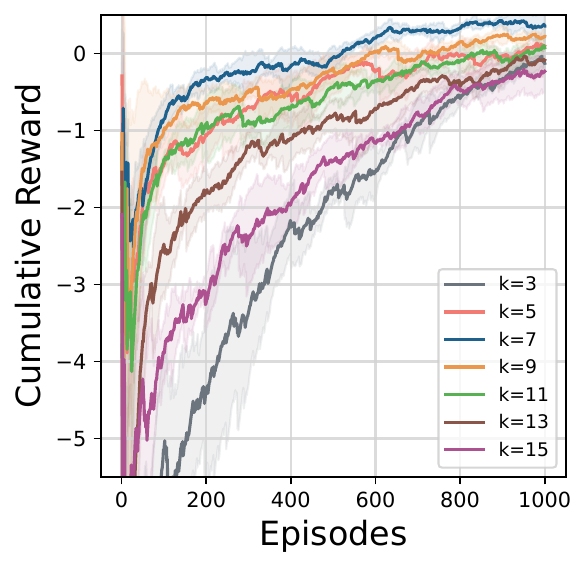}
    }
	\centering
    \subfloat[] {
        \label{fig:policy_steps_ablition}
        \centering
        \includegraphics[width=0.3\linewidth]{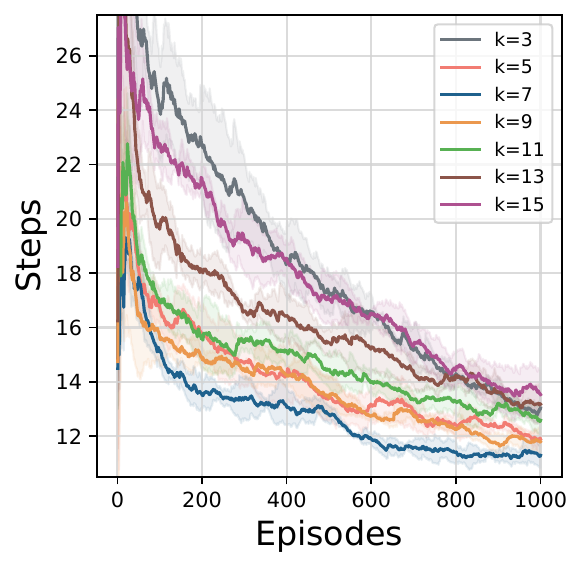}
    }
    
	\centering
    \subfloat[] {
        \label{fig:policy_alarms_ablition}
        \centering
        \includegraphics[width=0.3\linewidth]{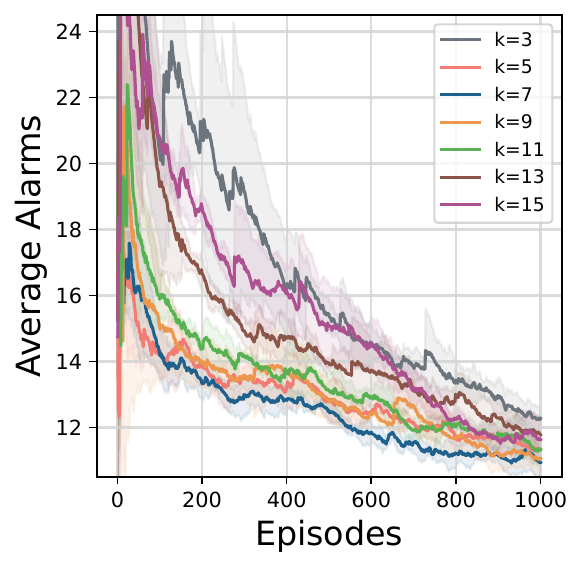}
    }
    \centering
    \subfloat[] {
        \label{fig:f1_ablition}
        \centering
        \includegraphics[width=0.3\linewidth]{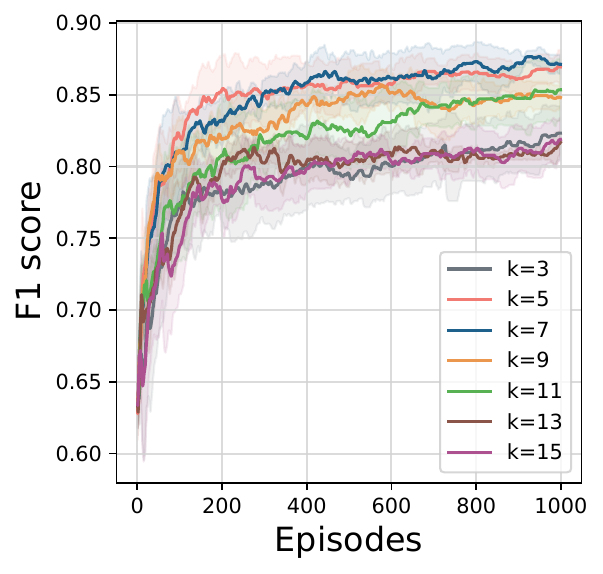}
    }
    \centering
    \caption{(a)-(c)Cumulative rewards, intervention steps, and average number of alarms per episode for Causal PPO based on THP initialization structures at different $K$ values; (d) The F1 scores of causal structure learning based on Causal PPO with THP initialization structure for different $K$ values.} \label{fig:exp_policy_result_ablation}
\end{figure}

\subsection{Analysis of causal structure learning}

To better demonstrate the effectiveness of our method, we only provide a small amount of observational data in the early causal structure learning.
As shown in Table~\ref{tab_causal}, the causal structure learned by THP in the initial stage has a large distance from the ground truth.
However, as we continue to interact with the environment, our methods gradually update the causal graph, bringing the learned causal structure closer to the ground truth. 
From Table~\ref{tab_causal} we can see that the F1 score values of our causal method are all over 0.8, which is significant compared with the initial THP result. The learned causal structures are given in Figure \ref{fig:exp_causal_result}. We can see that the proposed method can indeed identify the correct structure and interestingly all the root cause variables are mostly identified due to the identification of the causal order. 
In order to verify the robustness of our causal graph updating mechanism, we also conducted experiments on the initial random graph. As shown in Table~\ref{tab_causal}, even if the initial random graph is far from the ground truth, through continuous interactive updating, we can eventually learn a more accurate causal structure compared with the THP algorithm.
In addition, as shown in Figure \ref{fig:f1}, our methods converge to the optimal value early in the pre-training period for the learning of causal structure, regardless of whether it is given a random graph or a prior graph, which indicates that a small amount of intervention up front is enough to learn the causal structure.
Taking Causal PPO as an example, its F1 score has reached $0.7$ after only $20$ episodes.
This shows that even in the case of random initial causal structure, our method can still achieve a correct causal graph by calculating the treatment effects and performing the pruning step, which is more robust in the application.

\subsection{Sensitivity analysis}
The parameter $K$ represents the number of potential root-cause errors considered in the causal order. We further conduct sensitivity experiments to evaluate the sensitivity of the hyperparameter $K$, which controls the TopK causal order in policy learning. We conduct a sensitivity analysis using Causal PPO as a case study. 
The results are given in Figures~\ref{fig:policy_reward_ablition}~-~\ref{fig:f1_ablition}, which show the variations in the accuracy and robustness of policy learning and causal structure learning for different values of $K$. 
Specifically, when the $K$ is too large (e.g., $K>11$), the candidate the action under the causal mask would also be large, increasing the redundancy of the action space which decreases the policy's performance. Similarly, when the $K$ value is small (e.g., $K<5$), the policy's performance worsens because the overly constrained action space may limit the exploration of optimal actions. Thus, the $K$ controls the trade-off between the exploration and the exposition in our method.

\section{Conclusion}
This paper proposes an online causal reinforcement learning framework with a causal-aware policy that injects the causal structure into policy learning while devising a causal structure learning method by connecting the intervention and the action of the policy. We theoretically prove that our causal structure learning can identify the correct causal structure. To evaluate the performance of the proposed method, we constructed a FaultAlarmRL environment. Experiment results show that our method achieves accurate and robust causal structure learning as well as superior performance compared with SOTA baselines for policy learning.


\Acknowledgements{This research was supported in part by National Science and Technology Major Project (2021ZD0111501), National Science Fund for Excellent Young Scholars (62122022), Natural Science Foundation of China (U24A20233, 62206064, 62206061, 62476163, 62406078,62406080),  Guangdong Basic and Applied Basic Research Foundation (2023B1515120020).}


\begin{thebibliography}{10}
\expandafter\ifx\csname urlstyle\endcsname\relax
  \providecommand{\doi}[1]{doi:\discretionary{}{}{}#1}\else
  \providecommand{\doi}{doi:\discretionary{}{}{}\begingroup \urlstyle{rm}\Url}\fi

\bibitem{sutton2018reinforcement}
Sutton R~S and Barto A~G.
\newblock Reinforcement learning: An introduction.
\newblock Robotica, 1999.
\newblock 17(2):229--235

\bibitem{kober2013reinforcement}
Kober J, Bagnell J~A, and Peters J.
\newblock Reinforcement learning in robotics: A survey.
\newblock The International Journal of Robotics Research, 2013.
\newblock 32(11):1238--1274

\bibitem{silver2016mastering}
Silver D, Huang A, Maddison C~J, et~al.
\newblock Mastering the game of go with deep neural networks and tree search.
\newblock nature, 2016.
\newblock 529(7587):484--489

\bibitem{shalev2016safe}
Shalev-Shwartz S, Shammah S, and Shashua A.
\newblock Safe, multi-agent, reinforcement learning for autonomous driving.
\newblock arXiv preprint arXiv:161003295, 2016

\bibitem{sun2021model}
Sun Y, Zhang K, and Sun C.
\newblock Model-based transfer reinforcement learning based on graphical model representations.
\newblock {IEEE} Trans Neural Networks Learn Syst, 2023.
\newblock 34(2):1035--1048

\bibitem{zhu2022offline}
Zhu Z~M, Chen X~H, Tian H~L, et~al.
\newblock Offline reinforcement learning with causal structured world models.
\newblock arXiv preprint arXiv:220601474, 2022

\bibitem{sontakke2021causal}
Sontakke S~A, Mehrjou A, Itti L, et~al.
\newblock Causal curiosity: Rl agents discovering self-supervised experiments for causal representation learning.
\newblock In \emph{International Conference on Machine Learning}, volume 139. 2021.
\newblock 9848--9858

\bibitem{zhang2020learning}
Zhang A, McAllister R~T, Calandra R, et~al.
\newblock Learning invariant representations for reinforcement learning without reconstruction.
\newblock In \emph{9th International Conference on Learning Representations, {ICLR} 2021}, volume~9. 2021

\bibitem{tomar2021model}
Tomar M, Zhang A, Calandra R, et~al.
\newblock Model-invariant state abstractions for model-based reinforcement learning.
\newblock arXiv preprint arXiv:210209850, 2021

\bibitem{bica2021invariant}
Bica I, Jarrett D, and van~der Schaar M.
\newblock Invariant causal imitation learning for generalizable policies.
\newblock In \emph{Advances in Neural Information Processing Systems}, volume~34. 2021.
\newblock 3952--3964

\bibitem{sodhani2022improving}
Sodhani S, Levine S, and Zhang A.
\newblock Improving generalization with approximate factored value functions.
\newblock In \emph{ICLR2022 Workshop on the Elements of Reasoning: Objects, Structure and Causality}. 2022

\bibitem{wang2021task}
Wang Z, Xiao X, Zhu Y, et~al.
\newblock Task-independent causal state abstraction.
\newblock In \emph{Proceedings of the 35th International Conference on Neural Information Processing Systems, Robot Learning workshop}. 2021

\bibitem{DBLP:conf/nips/DingL0Z22}
Ding W, Lin H, Li B, et~al.
\newblock Generalizing goal-conditioned reinforcement learning with variational causal reasoning.
\newblock In \emph{Advances in Neural Information Processing Systems}, volume~35. 2022.
\newblock 26532--26548

\bibitem{seitzer2021causal}
Seitzer M, Sch{\"o}lkopf B, and Martius G.
\newblock Causal influence detection for improving efficiency in reinforcement learning.
\newblock Advances in Neural Information Processing Systems, 2021.
\newblock 34:22905--22918

\bibitem{huang2021adarl}
Huang B, Feng F, Lu C, et~al.
\newblock Adarl: What, where, and how to adapt in transfer reinforcement learning.
\newblock In \emph{The Tenth International Conference on Learning Representations, {ICLR}}, volume~10. 2022

\bibitem{huang2022action}
Huang B, Lu C, Leqi L, et~al.
\newblock Action-sufficient state representation learning for control with structural constraints.
\newblock In \emph{International Conference on Machine Learning}, volume 162. 2022.
\newblock 9260--9279

\bibitem{wang2021provably}
Wang L, Yang Z, and Wang Z.
\newblock Provably efficient causal reinforcement learning with confounded observational data.
\newblock In \emph{Advances in Neural Information Processing Systems}, volume~34. 2021.
\newblock 21164--21175

\bibitem{Luofeng2021}
Liao L, Fu Z, Yang Z, et~al.
\newblock Instrumental variable value iteration for causal offline reinforcement learning.
\newblock CoRR, 2021.
\newblock abs/2102.09907

\bibitem{Sergei2020}
Volodin S, Wichers N, and Nixon J.
\newblock Resolving spurious correlations in causal models of environments via interventions.
\newblock CoRR, 2020.
\newblock abs/2002.05217

\bibitem{AmyZhang2019}
Zhang A, Lipton Z~C, Pineda L, et~al.
\newblock Learning causal state representations of partially observable environments.
\newblock CoRR, 2019.
\newblock abs/1906.10437

\bibitem{lee2021causal}
Lee T~E, Zhao J~A, Sawhney A~S, et~al.
\newblock Causal reasoning in simulation for structure and transfer learning of robot manipulation policies.
\newblock In \emph{2021 IEEE International Conference on Robotics and Automation (ICRA)}. 2021.
\newblock 4776--4782

\bibitem{mnih2013playing}
Mnih V, Kavukcuoglu K, Silver D, et~al.
\newblock Playing atari with deep reinforcement learning.
\newblock arXiv preprint arXiv:13125602, 2013

\bibitem{lillicrap2015continuous}
Lillicrap T~P, Hunt J~J, Pritzel A, et~al.
\newblock Continuous control with deep reinforcement learning.
\newblock In Y~Bengio and Y~LeCun, editors, \emph{4th International Conference on Learning Representations, {ICLR} 2016, San Juan, Puerto Rico, May 2-4, 2016, Conference Track Proceedings}, volume~4. 2016

\bibitem{wang2016dueling}
Wang Z, Schaul T, Hessel M, et~al.
\newblock Dueling network architectures for deep reinforcement learning.
\newblock In \emph{International conference on machine learning}, volume~33. 2016.
\newblock 1995--2003

\bibitem{schulman2017proximal}
Schulman J, Wolski F, Dhariwal P, et~al.
\newblock Proximal policy optimization algorithms.
\newblock arXiv preprint arXiv:170706347, 2017

\bibitem{schulman2015trust}
Schulman J, Levine S, Abbeel P, et~al.
\newblock Trust region policy optimization.
\newblock In \emph{International conference on machine learning}, volume~32. 2015.
\newblock 1889--1897

\bibitem{mnih2016asynchronous}
Mnih V, Badia A~P, Mirza M, et~al.
\newblock Asynchronous methods for deep reinforcement learning.
\newblock In \emph{Proceedings of the 33nd International Conference on Machine Learning}, volume~33. 2016.
\newblock 1928--1937

\bibitem{haarnoja2018soft}
Haarnoja T, Zhou A, Abbeel P, et~al.
\newblock Soft actor-critic: Off-policy maximum entropy deep reinforcement learning with a stochastic actor.
\newblock In \emph{International conference on machine learning}, volume~80. 2018.
\newblock 1861--1870

\bibitem{kaiser2019model}
Kaiser L, Babaeizadeh M, Milos P, et~al.
\newblock Model-based reinforcement learning for atari.
\newblock arXiv preprint arXiv:190300374, 2019

\bibitem{sutton1991dyna}
Sutton R~S.
\newblock Dyna, an integrated architecture for learning, planning, and reacting.
\newblock ACM Sigart Bulletin, 1991.
\newblock 2(4):160--163

\bibitem{janner2019trust}
Janner M, Fu J, Zhang M, et~al.
\newblock When to trust your model: Model-based policy optimization.
\newblock Advances in Neural Information Processing Systems, 2019.
\newblock 32

\bibitem{garcia1989model}
Garcia C~E, Prett D~M, and Morari M.
\newblock Model predictive control: Theory and practice—a survey.
\newblock Automatica, 1989.
\newblock 25(3):335--348

\bibitem{luo2024survey}
Luo F~M, Xu T, Lai H, et~al.
\newblock {A survey on model-based reinforcement learning}.
\newblock Science China Information Sciences, 2024.
\newblock 67(2):121101

\bibitem{yanzeng2023}
Zeng Y, Cai R, Sun F, et~al.
\newblock A survey on causal reinforcement learning.
\newblock CoRR, 2023.
\newblock abs/2302.05209.
\newblock \doi{10.48550/arXiv.2302.05209}

\bibitem{de2019causal}
De~Haan P, Jayaraman D, and Levine S.
\newblock Causal confusion in imitation learning.
\newblock Advances in Neural Information Processing Systems, 2019.
\newblock 32

\bibitem{sonar2021invariant}
Sonar A, Pacelli V, and Majumdar A.
\newblock Invariant policy optimization: Towards stronger generalization in reinforcement learning.
\newblock In \emph{Proceedings of the 3rd Annual Conference on Learning for Dynamics and Control}, volume~3. 2021.
\newblock 21--33

\bibitem{lu2020sample}
Lu C, Huang B, Wang K, et~al.
\newblock Sample-efficient reinforcement learning via counterfactual-based data augmentation.
\newblock arXiv preprint arXiv:201209092, 2020

\bibitem{pitis2020counterfactual}
Pitis S, Creager E, and Garg A.
\newblock Counterfactual data augmentation using locally factored dynamics.
\newblock Advances in Neural Information Processing Systems, 2020.
\newblock 33:3976--3990

\bibitem{wang2022causal}
Wang Z, Xiao X, Xu Z, et~al.
\newblock Causal dynamics learning for task-independent state abstraction.
\newblock In \emph{International Conference on Machine Learning, {ICML}}, volume 162 of \emph{Proceedings of Machine Learning Research}. 2022.
\newblock 23151--23180

\bibitem{spirtes2000causation}
Spirtes P, Glymour C, and Scheines R.
\newblock \emph{Causation, prediction, and search}.
\newblock MIT press, 2001

\bibitem{chickering2002optimal}
Chickering D~M.
\newblock Optimal structure identification with greedy search.
\newblock Journal of machine learning research, 2002.
\newblock 3(Nov):507--554

\bibitem{ramsey2017million}
Ramsey J, Glymour M, Sanchez-Romero R, et~al.
\newblock A million variables and more: the fast greedy equivalence search algorithm for learning high-dimensional graphical causal models, with an application to functional magnetic resonance images.
\newblock International journal of data science and analytics, 2017.
\newblock 3(2):121--129

\bibitem{huang2018generalized}
Huang B, Zhang K, Lin Y, et~al.
\newblock Generalized score functions for causal discovery.
\newblock In \emph{Proceedings of the 24th ACM SIGKDD international conference on knowledge discovery \& data mining}, volume~24. 2018.
\newblock 1551--1560

\bibitem{shimizu2006linear}
Shimizu S, Hoyer P~O, Hyv{\"a}rinen A, et~al.
\newblock A linear non-gaussian acyclic model for causal discovery.
\newblock Journal of Machine Learning Research, 2006.
\newblock 7(10)

\bibitem{hoyer2008nonlinear}
Hoyer P, Janzing D, Mooij J~M, et~al.
\newblock Nonlinear causal discovery with additive noise models.
\newblock Advances in neural information processing systems, 2008.
\newblock 21

\bibitem{peters2014causal}
Peters J, Mooij J~M, Janzing D, et~al.
\newblock {Causal Discovery with Continuous Additive Noise Models}.
\newblock Journal of Machine Learning Research, 2014.
\newblock 15:2009--2053

\bibitem{cai2018self}
Cai R, Qiao J, Zhang Z, et~al.
\newblock {Self: structural equational likelihood framework for causal discovery}.
\newblock In \emph{Proceedings of the AAAI Conference on Artificial Intelligence}, volume~32. 2018.
\newblock 1787--1794

\bibitem{brouillard2020differentiable}
Brouillard P, Lachapelle S, Lacoste A, et~al.
\newblock Differentiable causal discovery from interventional data.
\newblock Advances in Neural Information Processing Systems, 2020.
\newblock 33:21865--21877

\bibitem{tigas2022interventions}
Tigas P, Annadani Y, Jesson A, et~al.
\newblock Interventions, where and how? experimental design for causal models at scale.
\newblock Advances in Neural Information Processing Systems, 2022.
\newblock 35:24130--24143

\bibitem{abel2022theory}
Abel D.
\newblock A theory of abstraction in reinforcement learning.
\newblock CoRR, 2022.
\newblock abs/2203.00397.
\newblock \doi{10.48550/arXiv.2203.00397}

\bibitem{peters2017elements}
Peters J, Janzing D, and Sch{\"o}lkopf B.
\newblock \emph{Elements of causal inference: foundations and learning algorithms}.
\newblock The MIT Press, 2017

\bibitem{pearl2009causality}
Pearl J.
\newblock \emph{Causality}.
\newblock Cambridge university press, 2009

\bibitem{rosenbaum1983central}
Rosenbaum P~R and Rubin D~B.
\newblock {The central role of the propensity score in observational studies for causal effects}.
\newblock Biometrika, 1983.
\newblock 70(1):41--55

\bibitem{lopez2017estimation}
Lopez M~J and Gutman R.
\newblock Estimation of causal effects with multiple treatments: A review and new ideas.
\newblock Statistical science, 2017.
\newblock 32(3):432--454

\bibitem{callaway2024difference}
Callaway B, Goodman-Bacon A, and Sant'Anna P~H.
\newblock Difference-in-differences with a continuous treatment.
\newblock Technical report, National Bureau of Economic Research, 2024

\bibitem{samsami2021causal}
Samsami M~R, Bahari M, Salehkaleybar S, et~al.
\newblock Causal imitative model for autonomous driving.
\newblock arXiv preprint arXiv:211203908, 2021

\bibitem{Susan_ATT}
Athey S, Imbens G~W, and Wager S.
\newblock {Approximate Residual Balancing: Debiased Inference of Average Treatment Effects in High Dimensions}.
\newblock Journal of the Royal Statistical Society Series B: Statistical Methodology, 2018.
\newblock 80(4):597--623.
\newblock ISSN 1369-7412.
\newblock \doi{10.1111/rssb.12268}

\bibitem{zheng2018dags}
Zheng X, Aragam B, Ravikumar P~K, et~al.
\newblock {Dags with no tears: Continuous optimization for structure learning}.
\newblock Advances in neural information processing systems, 2018.
\newblock 31

\bibitem{cai2022thps}
Cai R, Wu S, Qiao J, et~al.
\newblock Thps: Topological hawkes processes for learning causal structure on event sequences.
\newblock IEEE Transactions on Neural Networks and Learning Systems, 2022.
\newblock 35(1):479--493

\bibitem{cai2021thp}
Cai R, Wu S, Qiao J, et~al.
\newblock Thps: Topological hawkes processes for learning causal structure on event sequences.
\newblock {IEEE} Trans Neural Networks Learn Syst, 2024.
\newblock 35(1):479--493

\bibitem{xu2020error}
Xu T, Li Z, and Yu Y.
\newblock Error bounds of imitating policies and environments.
\newblock Advances in Neural Information Processing Systems, 2020.
\newblock 33:15737--15749

\end{thebibliography}

\newpage
\begin{appendix}
\setcounter{theorem}{0}
\section{Table of notation table}\label{sec_appdendix:notation}
Table \ref{tab:notation} summarizes notations used in this paper.

\begin{table}[h!]
\liuhao
\renewcommand\arraystretch{1.1}
\caption{A summary of the notation used in this paper.}
\label{tab:notation}
\centering
  \begin{tabular}{c|l}
    \toprule
    Notation             & Description  \\
    \midrule
    $\mathcal{S}$         & State space  \\
   $\mathcal{A}$         & Action space \\
    $\mathbf{s}$       & A vector of state in $\mathcal{S}$, i.e., $\mathbf{s}=[s_1,s_2,\dots,s_|\mathcal{s}|]$  \\
    $|\mathcal{S}|$& The number of states in the state space.\\
   $p(\mathbf{s}'|\mathbf{s},a)$            & the dynamic transition from state $\mathbf{s}\in \mathcal{S}$ to the next state $\mathbf{s}'$ when performing action $a\in \mathcal{A}$ in state $\mathbf{s}$  \\
    $r(\mathbf{s},a)$            & A reward on state $\mathbf{s}$ and action $a$ \\
    $\gamma$                & The discount factor \\
    $s_i$         & The $i$-th state variable.  \\
    $s_i^t$            & The $i$-th state variable at time $t$.  \\
    $V_{\mathcal{S}}$ & The vertex set on causal graph defined on the state variables\\
    $E$ & The causal edge set in the causal graph\\
    $\mathcal{G}$ & Causal graph that contains vertex $V_{\mathcal{S}}$ and edge set $E$\\
    $\mathbf{Pa}_i^{\mathcal{G}}$         & The parent set of $s_i$ in graph $\mathcal{G}$. \\
    $a_i$ & The action (treatment) on state $s_i$. \\
    $\mathbf{G}$ & The adjacency matrix of the causal graph.\\
    $\mathcal{C}_{s_i\to s_j}^{Att}$ & The average treatment effect for the treated sample from $s_i$ to $s_j$ when $s_i$ is treated. \\
    $\hat{\mathcal{C}}_{s_i\to s_j}^{Att}$ & The estimated ATT of $\mathcal{C}_{s_i\to s_j}^{Att}$. \\
    $M_{\mathbf{s}}(\mathcal{G})$ & The causal mask in the causal policy where $M_{\mathbf{s}}(\mathcal{G})=\{m_{\mathbf{s},a}^{\mathcal{G}}\}_{a=1}^{|\mathcal{A}|}$ \\
    $m_{\mathbf{s},a}^{\mathcal{G}}$ & The element of mask on action $a$ in the state $\mathbf{s}$ on causal graph $\mathcal{G}$\\
    $D_{TV}(\cdot,\cdot)$& Total variation distance.\\
    $V_{\pi_{\mathcal{G}}}$& The value function on policy $pi_{\mathcal{G}}$\\
$\mathbf{h}_{pi_{\mathcal{G}}}$& State distribution of causal policy $\pi_{\mathcal{G}}$\\
    $\mathbf{P}_{\pi _{\mathcal{G}}}( \mathbf{s}'|\mathbf{s})$& The $|\mathcal{S}|\times|\mathcal{S}|$ state matrix and its entry in $s',s$ where each present a probability from $s$ to $s'$ in policy $\pi_{\mathcal{G}}$\\
    $M_{\pi_{\mathcal{G}}}$& The $|\mathcal{S}|\times |\mathcal{A}||\mathcal{S}|$ transition matrix.\\
    $R_{\text{max}}$& The max reward.\\
    $A \ci_p B$ & Denote the statistical independence constraint between variables $A$ and $B$.\\
     $A \ci_p B \mid C$ & Denote the statistical conditional independence constraint between variables $A$ and $B$ conditioned on $C$.\\
    \bottomrule
  \end{tabular}
\end{table}

\section{Theoretical proofs}\label{sec_appdendix:theory}
\label{proof}

\subsection{Causal discovery}
In this section, we provide proof of the identifiability of causal order in the orientation step and the identifiability of causal structure after the pruning step. In identifying the causal order, we utilize the average treatment effect in treated (ATT) \cite{Susan_ATT} which can be written as follows:
\begin{equation}\label{eq:E_att}
        \begin{aligned}
            \mathcal{C}^{Att}_{s_i\to s_j}=\mathbb{E}[s_j(I_i=1)-s_j(I_i=0)|I_i=1],
        \end{aligned}
    \end{equation}
where $s_j(a_i=1)$ denotes the potential outcome of $s_j$ if $s_i$ were treated, $s_j(a_i=0)$ denotes the potential outcome if $s_i$ were not treated \cite{rosenbaum1983central}, and $\mathbb{E}$ denotes the expectation.

\begin{theorem}\label{thm:causal order_appendix}
    Given a causal graph $\mathcal{G}=(V_\mathcal{S},E)$, for each pair of states $s_i,s_j$ with $i\neq j$, $s_i$ is the ancestor of $s_j$ 
    if and only if $|\mathcal{C}^{Att}_{s_i\to s_j}|>0$.
\end{theorem}
\begin{proof}[Proof of Theorem \ref{thm:causal order_appendix}.]

$\displaystyle \Longrightarrow :$ If $s_{i}$ is the ancestor of $s_{j}$, then the intervention of $\displaystyle s_{i}$ will force manipulating the value of $\displaystyle s_{i}$ by definition and thus result in the change of $\displaystyle s_{j}$ compared with the $\displaystyle s_{j}$ without intervention. That is, $\displaystyle s_{j}( a_{i} =1) \neq s_{j}( I_{i} =0)$ and therefore $\displaystyle |s_{j}( I_{i} =1) -s_{j}( I_{i} =0) | >0$. By taking the average in population that is treated, we obtain $\displaystyle E[|s_{j} (I_{i} =1)-s_{j} (a_{i} =0)||I_{i} =1] >0$.

$\displaystyle \Longleftarrow :$ Similarly, if $|\mathcal{C}_{s_{i}\rightarrow s_{j}}^{Att} | >0$, we have $\displaystyle |s_{j} (I_{i} =1)-s_{j} (I_{i} =0)| >0$ based on Eq. \ref{eq:E_att}. To show $s_{i}$ is the ancestor of $s_{j}$, we prove by contradiction. Suppose $\displaystyle s_{i}$ is not the ancestor of $\displaystyle s_{j}$, then the intervention of $\displaystyle s_{i}$ will not change the value of $\displaystyle s_{j}$. That is, $\displaystyle s_{j} (I_{i} =1)=s_{j} (I_{i} =0)$ which creates the contradiction. Thus, $\displaystyle s_{i}$ is the ancestor of $\displaystyle s_{j}$ which finishes the proof. \qedsymbol
\end{proof}

The following theorem shows that the causal structure is identifiable given the correct causal order. The overall proof is built based on \cite{chickering2002optimal}. The main idea is that the causal structure can be identified given the correct causal order if we can identify the causal skeleton. To learn the causal skeleton, we can resort to identifying the (conditional) independence among the variables. Thus, in the following, we will show that under the causal Markov assumption, faithfulness assumption and the sufficiency assumption, the (conditional) independence of the variables can be identified by the proposed BIC score in our work due to its locally consistent property. We begin with the definition of the locally consistent scoring criterion.

\begin{definition}[Locally consistent scoring criterion]\label{def:consistent}
    Let $D$ be a set of data consisting of $m$ records that are iid samples from some distribution $p(\cdot)$. Let $\mathcal{G}$ be any $D A G$, and let $\mathcal{G}^{\prime}$ be the $D A G$ that results from adding the edge $X_i \rightarrow X_j$. A scoring criterion $S(\mathcal{G}, D)$ is locally consistent if in the limit as $m$ grows large the following two properties hold:
    \begin{enumerate}
        \item If $X_j \nci_p X_i \mid X_{\mathbf{Pa}_j^{\mathcal{G}}}$, then $S(\mathcal{G}^{\prime}, D)>S(\mathcal{G}, D)$.
        \item If $X_j \ci_p X_i \mid X_{\mathbf{Pa}_j^{\mathcal{G}}}$, then $S(\mathcal{G}^{\prime}, D)<S(\mathcal{G}, D)$.
    \end{enumerate}
\end{definition}
\begin{lemma}[Lemma 7 in \cite{chickering2002optimal}] \label{lem:bic}
    The Bayesian scoring criterion (BIC) is locally consistent.
\end{lemma}
Note that, as pointed out by \cite{chickering2002optimal}, the BIC, which can be rewritten as the $\ell_0$-norm penalty as Eq. (6) in the main text, is locally consistent. This property allows us to correctly identify the independence relationship among states by using the locally consistent BIC score because we can always obtain a greater score if the searched graph consists of (conditional) independence in the data. Thus, we can always search a causal graph $\mathcal{G}$ with the highest score that is `correct' in the sense that all (conditional) independence consists of the ground truth. This is concluded by the following theorem:
\begin{theorem}[Identifiability]\label{thm:identifiability_appendix}
    Under the causal faithfulness and causal sufficiency assumptions, given the correct causal order and large enough data, the causal structure among states is identifiable from observational data.
\end{theorem}
\begin{proof}[Proof of Theorem \ref{thm:identifiability_appendix}]
    Based on Lemma \ref{lem:bic}, Eq. (6) in the main text is locally consistent since it has the same form of the BIC score and we denote it using $S(\mathcal{G},D)$. Then we can prune the redundant edge if $S(\mathcal{G}',D)>S(\mathcal{G},D)$ where $\mathcal{G}'$ is the graph that removes one of the redundant edges. The reason is that for any pair of state $s_i, s_j$ is redundant, there must exist a conditional set $\mathbf{Pa}^\mathcal{G}(s_j)$ such that $s_i \ci s_j \mid Pa_\mathcal{G}(s_j)$. Then based on the second property in Definition \ref{def:consistent}, we have $S(\mathcal{G}',D)>S(\mathcal{G},D)$ since $\mathcal{G}$ can be seen as the graph that adds a redundant edge from $\mathcal{G}'$. Moreover, since we have causal faithfulness and causal sufficiency assumptions, such independence will be faithful to the causal graph, and thus, by repeating the above step, we are able to obtain the correct causal structure.
\end{proof}

\subsection{Policy performance guarantee}
In this section, we provide the policy performance guarantees step by step.
We first recap the causal policy in the following definition:
\begin{customdef}{B1}[Causal policy]\label{def:causal_policy_appendix}
    Given a causal graph $\mathcal{G}$, we define the causal policy $\pi_{\mathcal{G}}(\cdot|\mathbf{s})$ under the causal graph $\mathcal{G}$ as follows:
    \begin{equation}
    \pi_{\mathcal{G}}(\cdot|\mathbf{s})=M_{\mathbf{s}}(\mathcal{G})\circ\pi(\cdot|\mathbf{s}),
    \end{equation}
    where $M_{\mathbf{s}}(\mathcal{G})$ is the causal mask vector at state $\mathbf{s}$ under the causal graph $\mathcal{G}$, and $\pi(\cdot|\mathbf{s})$ is the action probability distribution of the original policy output.
\end{customdef}
For example, the causal mask $M_{\mathbf{s}}(\mathcal{G})=\{m_{\mathbf{s},a}^{\mathcal{G}}\}_{a=1}^{|\mathcal{A}|}$ constitute the vector of mask $m_{\mathbf{s},a}^{\mathcal{G}}\in \{0,1\}$ of each action in $\mathcal{A}$ where $|\mathcal{A}|$ denotes the number of actions in the action space.

\textit{Outline of the proof of Theorem 3.} Our goal is to show that under the causal policy, the value function under the correct causal graph will have greater value than the value function that has misspecified causal graph such that the differences of the value function can be bound by some constant $\displaystyle c >0$:
\begin{equation}
V_{\pi _{\mathcal{G}^{*}}} -V_{\pi _{\mathcal{G}}} \leqslant c.
\end{equation}
To do so, one may first notice that the difference of the value function can be expressed and bounded by the total variation $\displaystyle D_{\mathrm{TV}} (\rho _{\pi _{\mathcal{G}}} ,\rho _{\pi _{\mathcal{G}^{*}}} )$:
\begin{equation}
|V_{\pi _{\mathcal{G}^{*}}} -V_{\pi _{\mathcal{G}}} |\leq \frac{2R_{\max}}{1-\gamma } D_{\mathrm{TV}} (\rho _{\pi _{\mathcal{G}}} ,\rho _{\pi _{\mathcal{G}^{*}}} ). \label{eq:outline_v}
\end{equation}
Such a total variation can be further bound by the total variation of $\displaystyle D_{\mathrm{TV}} (\pi _{\mathcal{G}} (\cdot \mid \mathbf{s} ),\pi _{\mathcal{G}^{*}} (\cdot \mid \mathbf{s} ))$ (Lemma \ref{lem:state_discrepancy} and Lemma \ref{lem:state_action_discrepancy}):
\begin{equation}
D_{\mathrm{TV}} (\rho _{\pi _{\mathcal{G}}} ,\rho _{\pi _{\mathcal{G}^{*}}} )\leq \dfrac{1}{1-\gamma }\mathbb{E}_{\mathbf{s} \sim \mathbf{h}_{\pi _{\mathcal{G} *}}} [D_{\mathrm{TV}} (\pi _{\mathcal{G}} (\cdot \mid \mathbf{s} ),\pi _{\mathcal{G}^{*}} (\cdot \mid \mathbf{s} ))].\label{eq:outline_p}
\end{equation}
Combining Eq. \ref{eq:outline_v} and Eq. \ref{eq:outline_p}, we have
\begin{equation}
|V_{\pi _{\mathcal{G}^{*}}} -V_{\pi _{\mathcal{G}}} |\leq \dfrac{2R_{\max}}{( 1-\gamma )^{2}}\mathbb{E}_{\mathbf{s} \sim \mathbf{h}_{\pi _{\mathcal{G} *}}} [D_{\mathrm{TV}} (\pi _{\mathcal{G}} (\cdot \mid \mathbf{s} ),\pi _{\mathcal{G}^{*}} (\cdot \mid \mathbf{s} ))].\label{eq:outline_vp}
\end{equation}
By this, we can delve into this bound by investigating the total variation of the causal policy. Based on the definition of the causal policy in Definition \ref{def:causal_policy_appendix}. One can deduce that the distance should be related to the difference of the causal mask, and it is true that as shown in Lemma \ref{lem:policy_discrepancy}:
\begin{equation}
\begin{aligned}
D_{TV} (\pi _{\mathcal{G}^{*}} ,\pi _{\mathcal{G}} )\leq \dfrac{1}{2}\left( \| M_{\mathbf{s}} (\mathcal{G} )-M_{\mathbf{s}}\left(\mathcal{G}^{*}\right) \| _{1} +\| \mathbf{1}_{\left\{a :m_{\mathbf{s} ,a}^{\mathcal{G}^{*}} =1\land m_{\mathbf{s} ,a}^{\mathcal{G}} =1\right\}} \| _{1}\right).
\end{aligned} \label{eq:outline_tvp}
\end{equation}
Finally, by combining Eq. \ref{eq:outline_vp} and Eq. \ref{eq:outline_tvp} and further due to the positive of the bound, we obtain the result in Theorem \ref{thm:value_bound_appendix}:
\begin{equation}
\begin{aligned}
V_{\pi _{\mathcal{G}^{*}}} -V_{\pi _{\mathcal{G}}} \leq  & \frac{R_{\max}}{(1-\gamma )^{2}} (\| M_{\mathbf{s}} (\mathcal{G} )-M_{\mathbf{s}} (\mathcal{G}^{*} )\| _{1}\\
 & +\| \mathbf{1}_{\{a :m_{\mathbf{s} ,a}^{\mathcal{G}^{*}} =1\land m_{\mathbf{s} ,a}^{\mathcal{G}} =1\}} \| _{1} ).
\end{aligned}
\end{equation}

With the outline above, in the following, we provide the details proof of the Lemma \ref{lem:state_discrepancy}, Lemma \ref{lem:state_action_discrepancy}, Lemma \ref{lem:policy_discrepancy}, and Theorem \ref{thm:value_bound_appendix}, respectively.

\begin{lemma}\label{lem:policy_discrepancy}
Let $\pi_{\mathcal{G}^{*}}(\cdot |\mathbf{s})$ be the policy under the true causal graph $\mathcal{G}^{*}=\left(V_\mathcal{S},E^{*}\right)$. For any causal graph $\mathcal{G}=(V_\mathcal{S},E)$, when the defined causal policy $\pi_{\mathcal{G}}(\cdot |\mathbf{s})$ converges, the following inequality holds: 
\begin{equation}
\begin{split}
D_{TV}(\pi_{\mathcal{G}^{*}},\pi_{\mathcal{G}}) \leq \dfrac{1}{2} (\|M_{\mathbf{s}}(\mathcal{G})-M_{\mathbf{s}}(\mathcal{G}^{*})\|_{1}  + \Vert\mathbf{1}_{\{a:m_{\mathbf{s},a}^{\mathcal{G}^{*}}=1\land m_{\mathbf{s},a}^{\mathcal{G}}=1 \} } \Vert_{1},
\end{split}
\end{equation}
where $\displaystyle \| M_{\mathbf{s}} (\mathcal{G} )-M_{\mathbf{s}}\left(\mathcal{G}^{*}\right) \| _{1}$ is the $\displaystyle \ell _{1}$-norm of the masks measuring the differences of two policies, $\displaystyle \mathbf{1}$ is an indicator function and $\displaystyle \| \mathbf{1}_{\left\{a:m_{\mathbf{s},a}^{\mathcal{G}^{*}} =1\land m_{\mathbf{s},a}^{\mathcal{G}} =1\right\}} \| _{1}$ measures the number of actions that are not masked on both policies.
\end{lemma}

\begin{proof}[Proof of Lemma \ref{lem:policy_discrepancy}]
Based on the definition of the total variation and the causal policy we have:
\begin{equation}
\begin{aligned}
 D_{TV} (\pi _{\mathcal{G}^{*}} ,\pi _{\mathcal{G}} )
 & =\dfrac{1}{2} \| \pi _{\mathcal{G}^{*}} (\cdot |\mathbf{s})-\pi _{\mathcal{G}} (\cdot |\mathbf{s})\| _{1}\\
 & =\dfrac{1}{2}\sum _{a} |\pi _{\mathcal{G}^{*}} (a|\mathbf{s})-\pi _{\mathcal{G}} (a|\mathbf{s})|\\
 & =\dfrac{1}{2}\sum _{a} |m_{\mathbf{s},a}^{\mathcal{G}^{*}} \pi ^{*} (a|\mathbf{s})-m_{\mathbf{s},a}^{\mathcal{G}} \pi (a|\mathbf{s})|.
\end{aligned}
\end{equation}
Since the mask only takes value in $\displaystyle \{0,1\}$, we can rearrange the summation by considering the different values of the mask on the two policies:
\begin{equation}
\begin{aligned}
 D_{TV} (\pi _{\mathcal{G}^{*}} ,\pi _{\mathcal{G}} ) =\dfrac{1}{2}\left(\sum _{a:m_{\mathbf{s},a}^{\mathcal{G}^{*}} =1\land m_{\mathbf{s},a}^{\mathcal{G}} =0} |\pi ^{*} (a|\mathbf{s})|+\sum _{a:m_{\mathbf{s},a}^{\mathcal{G}^{*}} =0\land m_{\mathbf{s},a}^{\mathcal{G}} =1} |\pi (a|\mathbf{s})|+\sum _{a:m_{\mathbf{s},a}^{\mathcal{G}^{*}} =1\land m_{\mathbf{s},a}^{\mathcal{G}} =1} |\pi ^{*} (a|\mathbf{s})-\pi (a|\mathbf{s})|\right),
\end{aligned}
\end{equation}
where the summation when $\displaystyle m_{\mathbf{s},a}^{\mathcal{G}^{*}} =0\land m_{\mathbf{s},a}^{\mathcal{G}} =0$ is zero as policy on both side are masked out. Then, based on the fact that $\displaystyle 0\leq \pi (a|\mathbf{s})\leq 1$ of the policy, we have the following inequality
\begin{equation}
\begin{aligned}
 D_{TV} (\pi _{\mathcal{G}^{*}} ,\pi _{\mathcal{G}} ) \leq \dfrac{1}{2}\left( \| M_{\mathbf{s}} (\mathcal{G} )-M_{\mathbf{s}}\left(\mathcal{G}^{*}\right) \| _{1} +\| \mathbf{1}_{\left\{a:m_{\mathbf{s},a}^{\mathcal{G}^{*}} =1\land m_{\mathbf{s},a}^{\mathcal{G}} =1\right\}} \| _{1}\right).
\end{aligned}
\end{equation}
\end{proof}

Then we introduce the following Lemma~\ref{lem:state_discrepancy}, which bound the state distribution discrepancy based on the causal policy discrepancy.

\begin{lemma}\label{lem:state_discrepancy}
Given a policy $\pi _{\mathcal{G}^{*}}( \cdot |\mathbf{s})$ under the true causal structure $\mathcal{G}^{*} =\left( V,E^{*}\right)$ and an policy $\pi _{\mathcal{G}}( \cdot |\mathbf{s})$ under the causal graph $\mathcal{G} =( V,E)$ , we have that
\begin{equation}
\begin{aligned}
 D_{TV}(\mathbf{h}_{\pi_{\mathcal{G}}}(\mathbf{s}),\mathbf{h}_{\pi_{\mathcal{G}^{*}}}( \mathbf{s})) \leq \dfrac{1}{1-\gamma}\mathbb{E}_{\mathbf{s}\sim \mathbf{h}_{\pi_{\mathcal{G}*}}}[ D_{\mathrm{TV}}(\pi_{\mathcal{G}} (\cdot \mid \mathbf{s}),\pi _{\mathcal{G}^{*}} (\cdot \mid \mathbf{s}))].
\end{aligned}
\end{equation}
\end{lemma}

\begin{proof}[Proof of Lemma \ref{lem:state_discrepancy}]
The proof is inspired by~\cite{xu2020error}, we show that the state distribution $\mathbf{h}_{\pi_{\mathcal{G}}}$ of causal policy $\pi_{\mathcal{G}}$ can be denoted as
\begin{equation}
\mathbf{h}_{\pi _{\mathcal{G}}} =(1-\gamma)(I-\gamma \mathbf{P}_{\pi _{\mathcal{G}}})^{-1} \mathbf{h}_{0},
\end{equation}
where $\displaystyle \mathbf{P}_{\pi _{\mathcal{G}}}( \mathbf{s}'|\mathbf{s}) =\sum _{a\in \mathcal{A}} M^{*}( \mathbf{s}'\mid \mathbf{s},a) \pi _{\mathcal{G}} (a\mid \mathbf{s})$, and $\displaystyle M^{*}( \mathbf{s}'\mid \mathbf{s},a)$ is the dynamic model. 
Denote that $\displaystyle M_{\pi _{\mathcal{G}}} =( I-\gamma \mathbf{P}_{\pi _{\mathcal{G}}})^{-1}$, we then have
\begin{equation}
\begin{aligned}
\mathbf{h}_{\pi _{\mathcal{G}}} -\mathbf{h}_{\pi _{\mathcal{G}^{*}}} & =( 1-\gamma )\left[( I-\gamma \mathbf{P}_{\pi _{\mathcal{G}}})^{-1} -( I-\gamma \mathbf{P}_{_{\pi _{\mathcal{G}^{*}}}})^{-1}\right] \mathbf{h}_{0}\\
 & =( 1-\gamma )( M_{\pi _{\mathcal{G}}} -M_{\pi _{\mathcal{G}^{*}}}) \mathbf{h}_{0}\\
 & =( 1-\gamma ) \gamma M_{\pi _{\mathcal{G}}}( \mathbf{P}_{\pi _{\mathcal{G}}} -\mathbf{P}_{\pi _{\mathcal{G}^{*}}}) M_{\pi _{\mathcal{G}^{*}}} \mathbf{h}_{0}\\
 & =\gamma M_{\pi _{\mathcal{G}}}( \mathbf{P}_{\pi _{\mathcal{G}}} -\mathbf{P}_{\pi _{\mathcal{G}^{*}}}) \mathbf{h}_{\pi _{\mathcal{G}^{*}}}.
\end{aligned}
\end{equation}
Similarly to Lemma 4 in ~\cite{xu2020error}, we have
\begin{equation}
\begin{aligned}
D_{TV}( \mathbf{h}_{\pi _{\mathcal{G}}}(\mathbf{s}) ,\mathbf{h}_{\pi _{\mathcal{G}^{*}}}(\mathbf{s})) & =\dfrac{\gamma }{2} \| M_{\pi _{\mathcal{G}}}( \mathbf{P}_{\pi _{\mathcal{G}}} -\mathbf{P}_{\pi _{\mathcal{G}^{*}}}) \mathbf{h}_{\pi _{\mathcal{G}^{*}}} \| _{1}\\
 & \leq \dfrac{\gamma }{2} \| M_{\pi _{\mathcal{G}}} \| _{1} \| ( \mathbf{P}_{\pi _{\mathcal{G}}} -\mathbf{P}_{\pi _{\mathcal{G}^{*}}}) \mathbf{h}_{\pi _{\mathcal{G}^{*}}} \| _{1}.
\end{aligned}
\end{equation}
Note that
\begin{equation}
\| M_{\pi _{\mathcal{G}}} \| _{1} =\| \sum _{t=0}^{\infty } \gamma ^{t} \mathbf{P}_{\pi _{\mathcal{G}}}^{t} \| _{1} \leq \sum _{t=0}^{\infty } \gamma ^{t} \| \mathbf{P}_{\pi _{\mathcal{G}}} \| _{1}^{t} \leq \sum _{t=0}^{\infty } \gamma ^{t} =\dfrac{1}{1-\gamma },
\end{equation}
and we also show that $\displaystyle \| ( \mathbf{P}_{\pi _{\mathcal{G}}} -\mathbf{P}_{\pi _{\mathcal{G}^{*}}}) \mathbf{h}_{\pi _{\mathcal{G}^{*}}} \| _{1}$ is bounded by
\begin{equation}
\begin{aligned}
\| ( \mathbf{P}_{\pi _{\mathcal{G}}} -\mathbf{P}_{\pi _{\mathcal{G}^{*}}}) \mathbf{h}_{\pi _{\mathcal{G}^{*}}} \| _{1} & \leq \sum _{\mathbf{s},\mathbf{s}'} |\mathbf{P}_{\pi _{\mathcal{G}}}( \mathbf{s}'|\mathbf{s}) -\mathbf{P}_{\pi _{\mathcal{G}^{*}}}( \mathbf{s}'|\mathbf{s}) |\mathbf{h}_{\pi _{\mathcal{G}^{*}}}(\mathbf{s})\\
 & =\sum _{\mathbf{s},\mathbf{s}}\left| \sum _{a\in \mathcal{A}} M^{*}( \mathbf{s}\mid \mathbf{s},a)( \pi _{\mathcal{G}} (a\mid \mathbf{s})-\pi _{\mathcal{G}^{*}} (a\mid \mathbf{s}))\right| \mathbf{h}_{\pi _{\mathcal{G}^{*}}}( \mathbf{s})\\
 & \leq \sum _{( \mathbf{s},a) ,\mathbf{s}} M^{*}( \mathbf{s}\mid \mathbf{s},a)| \pi _{\mathcal{G}} (a\mid \mathbf{s})-\pi _{\mathcal{G}^{*}} (a\mid \mathbf{s})| \mathbf{h}_{\pi _{\mathcal{G}^{*}}}(\mathbf{s})\\
 & =\sum _{s} \mathbf{h}_{\pi _{\mathcal{G}^{*}}}(\mathbf{s})\sum _{a\in \mathcal{A}}| \pi _{\mathcal{G}} (a\mid \mathbf{s})-\pi _{\mathcal{G}^{*}} (a\mid \mathbf{s})| \\
 & =2\mathbb{E}_{\mathbf{s}\sim \mathbf{h}_{\pi _{\mathcal{G} *}}}[ D_{\mathrm{TV}}( \pi _{\mathcal{G}} (\cdot \mid \mathbf{s}),\pi _{\mathcal{G}^{*}} (\cdot \mid \mathbf{s}))].
\end{aligned}
\end{equation}
Thus, we have 
\begin{equation}
\begin{aligned}
D_{TV}( \mathbf{h}_{\pi _{\mathcal{G}}}(\mathbf{s}) ,\mathbf{h}_{\pi _{\mathcal{G}^{*}}}(\mathbf{s})) & \leq \dfrac{\gamma }{2} \| M_{\pi _{\mathcal{G}}} \| _{1} \| ( \mathbf{P}_{\pi _{\mathcal{G}}} -\mathbf{P}_{\pi _{\mathcal{G}^{*}}}) \mathbf{h}_{\pi _{\mathcal{G}^{*}}} \| _{1}\\
 & \leq \dfrac{1}{1-\gamma }\mathbb{E}_{\mathbf{s}\sim \mathbf{h}_{\pi _{\mathcal{G} *}}}[ D_{\mathrm{TV}}( \pi _{\mathcal{G}} (\cdot \mid \mathbf{s}),\pi _{\mathcal{G}^{*}} (\cdot \mid \mathbf{s}))].
\end{aligned}
\end{equation}
\end{proof}
Next, we further bound the state-action distribution discrepancy based on the causal policy discrepancy.
\begin{lemma}\label{lem:state_action_discrepancy}
Given a policy $\pi _{\mathcal{G}^{*}}( \cdot |\mathbf{s})$ under the true causal structure $\mathcal{G}^{*} =\left( V,E^{*}\right)$ and an policy $\pi _{\mathcal{G}}( \cdot |\mathbf{s})$ under the causal graph $\mathcal{G} =(V,E)$ , we have that
\begin{equation}
\begin{aligned}
D_{\mathrm{TV}}( \rho _{\pi _{\mathcal{G}}} ,\rho _{\pi _{\mathcal{G}^{*}}}) & \leq \dfrac{1}{1-\gamma }\mathbb{E}_{\mathbf{s}\sim \mathbf{h}_{\pi _{\mathcal{G} *}}}[ D_{\mathrm{TV}}( \pi _{\mathcal{G}} (\cdot \mid \mathbf{s}),\pi _{\mathcal{G}^{*}} (\cdot \mid \mathbf{s}))].
\end{aligned}
\end{equation}
\end{lemma}

\begin{proof}[Proof of Lemma \ref{lem:state_action_discrepancy}]
Note that for any policy $\displaystyle \pi _{\mathcal{G}}$ under any causal graph $\displaystyle \mathcal{G}$, the state-action distribution $\displaystyle \rho _{\pi _{\mathcal{G}}}( \mathbf{s},a) =\pi _{\mathcal{G}} (a\mid \mathbf{s})\mathbf{h}_{\pi _{\mathcal{G}}}(\mathbf{s})$, we have
\begin{equation}
\begin{aligned}
D_{\mathrm{TV}}( \rho _{\pi _{\mathcal{G}}} ,\rho _{\pi _{\mathcal{G}^{*}}}) & =\dfrac{1}{2}\sum _{( \mathbf{s},a)}| [ \pi _{\mathcal{G}^{*}} (a\mid \mathbf{s})-\pi _{\mathcal{G}} (a\mid \mathbf{s})] \mathbf{h}_{\pi _{\mathcal{G}}}(\mathbf{s}) +[ \mathbf{h}_{\pi _{\mathcal{G}^{*}}}(\mathbf{s}) -\mathbf{h}_{\pi _{\mathcal{G}}}(\mathbf{s})] \pi _{\mathcal{G}} (a\mid \mathbf{s})| \\
 & \leq \dfrac{1}{2}\sum _{( \mathbf{s},a)} |\pi _{\mathcal{G}^{*}} (a\mid \mathbf{s})-\pi _{\mathcal{G}} (a\mid \mathbf{s})| \mathbf{h}_{\pi _{\mathcal{G}}}(\mathbf{s}) +\dfrac{1}{2}\sum _{( \mathbf{s},a)} \pi _{\mathcal{G}} (a\mid \mathbf{s})|\mathbf{h}_{\pi _{\mathcal{G}^{*}}}(\mathbf{s}) -\mathbf{h}_{\pi _{\mathcal{G}}}(\mathbf{s}) |\\
 & =\mathbb{E}_{\mathbf{s}\sim \mathbf{h}_{\pi _{\mathcal{G} *}}}[ D_{\mathrm{TV}}( \pi _{\mathcal{G}} (\cdot \mid \mathbf{s}),\pi _{\mathcal{G}^{*}} (\cdot \mid \mathbf{s}))] +D_{TV}( \mathbf{h}_{\pi _{\mathcal{G}}}(\mathbf{s}) ,\mathbf{h}_{\pi _{\mathcal{G}^{*}}}(\mathbf{s}))\\
 & \leq \dfrac{1}{1-\gamma }\mathbb{E}_{\mathbf{s}\sim \mathbf{h}_{\pi _{\mathcal{G} *}}}[ D_{\mathrm{TV}}( \pi _{\mathcal{G}} (\cdot \mid \mathbf{s}),\pi _{\mathcal{G}^{*}} (\cdot \mid \mathbf{s}))],
\end{aligned}
\end{equation}
where the last inequality follows Lemma \ref{lem:state_discrepancy}.
\end{proof}
Based on all the above Lemma \ref{lem:state_action_discrepancy}, we finally give the policy performance guarantee of our proposed framework.
Specifically, we bound the policy value gap (i.e., the difference between the value of learned causal policy and the optimal policy) based on the state-action distribution discrepancy.
\begin{theorem}\label{thm:value_bound_appendix}
Given a causal policy $\pi_{\mathcal{G}^{*}}(\cdot|\mathbf{s})$ under the true causal graph $\mathcal{G}^{*}=\left(V_\mathcal{S},E^{*}\right)$ and a policy $\pi_{\mathcal{G}}(\cdot |\mathbf{s})$ under the causal graph $\mathcal{G} =(V_\mathcal{S},E)$, recalling $R_{\max}$ is the upper bound of the reward function, we have the performance difference of $\pi_{\mathcal{G}^{*}}(\cdot|\mathbf{s})$ and $\pi_{\mathcal{G}}(\cdot|\mathbf{s})$ be bounded as below,
\begin{equation}
\begin{split}
\begin{aligned}
V_{\pi_{\mathcal{G}^{*}}} -V_{\pi_{\mathcal{G}}}  \leq &\frac{R_{\max}}{(1-\gamma )^{2}}(\| M_{\mathbf{s}}(\mathcal{G})-M_{\mathbf{s}}(\mathcal{G}^{*}) \|_{1} \\& + \Vert \mathbf{1}_{\{a:m_{\mathbf{s},a}^{\mathcal{G}^{*}}=1\land m_{\mathbf{s},a}^{\mathcal{G}}=1 \}}\Vert_{1}).
\end{aligned}
\end{split}
\end{equation}
\end{theorem}

\begin{proof}[Proof of theorem \ref{thm:value_bound_appendix}]

Note that for any policy $\displaystyle \pi _{\mathcal{G}}$ under any causal graph $\displaystyle \mathcal{G}$, its policy value can be reformulated as $\displaystyle V_{\pi _{\mathcal{G}}} =\dfrac{1}{1-\gamma }\mathbb{E}_{( \mathbf{s},a) \sim \rho _{\pi _{\mathcal{G}}}}[ r,a]$. Based on this, we have
\begin{equation}
\begin{aligned}
|V_{\pi _{\mathcal{G}^{*}}} -V_{\pi _{\mathcal{G}}} | & =\left| \dfrac{1}{1-\gamma }\mathbb{E}_{( \mathbf{s},a) \sim \rho _{\pi _{\mathcal{G}}}}[ r,a] -\dfrac{1}{1-\gamma }\mathbb{E}_{( \mathbf{s},a) \sim \rho _{\pi _{\mathcal{G}^{*}}}}[ r,a]\right| \\
 & \leq \dfrac{1}{1-\gamma }\sum _{( \mathbf{s},a) \in \mathcal{S} \times \mathcal{A}}| ( \rho _{\pi _{\mathcal{G}}}( \mathbf{s},a) -\rho _{\pi _{\mathcal{G}^{*}}}( \mathbf{s},a)) r( \mathbf{s},a)| \\
 & \leq \frac{2R_{\max}}{1-\gamma } D_{\mathrm{TV}}( \rho _{\pi _{\mathcal{G}}} ,\rho _{\pi _{\mathcal{G}^{*}}}).
\end{aligned}
\end{equation}
Combining Lemma \ref{lem:state_action_discrepancy} and Lemma \ref{lem:policy_discrepancy}, we have
\begin{equation}
\begin{aligned}
V_{\pi _{\mathcal{G}^{*}}} -V_{\pi _{\mathcal{G}}} & \leq \frac{2R_{\max}}{1-\gamma } D_{\mathrm{TV}}( \rho _{\pi _{\mathcal{G}}} ,\rho _{\pi _{\mathcal{G}^{*}}})\\
 & \leq \frac{2R_{\max}}{(1-\gamma )^{2}}\mathbb{E}_{\mathbf{s}\sim d_{\pi _{\mathcal{G} *}}}[ D_{\mathrm{TV}}( \pi _{\mathcal{G}} (\cdot \mid \mathbf{s}),\pi _{\mathcal{G}^{*}} (\cdot \mid \mathbf{s}))]\\
 & \leq \frac{R_{\max}}{(1-\gamma )^{2}}\left( \| M_{\mathbf{s}}(\mathcal{G}) -M_{\mathbf{s}}\left(\mathcal{G}^{*}\right) \| _{1} +\Vert \mathbf{1}_{\left\{a:m_{\mathbf{s},a}^{\mathcal{G}^{*}} =1\land m_{\mathbf{s},a}^{\mathcal{G}} =1\right\}}\Vert _{1}\right),
\end{aligned}
\end{equation}
which completes the proof.
\end{proof}

\bigskip

\section{Additional experiment of topology-free environment}

\begin{figure}[t]
	\centering
    \subfloat[] {
        \label{fig:policy_reward_notopo}
        \centering
        \includegraphics[width=0.31\linewidth]{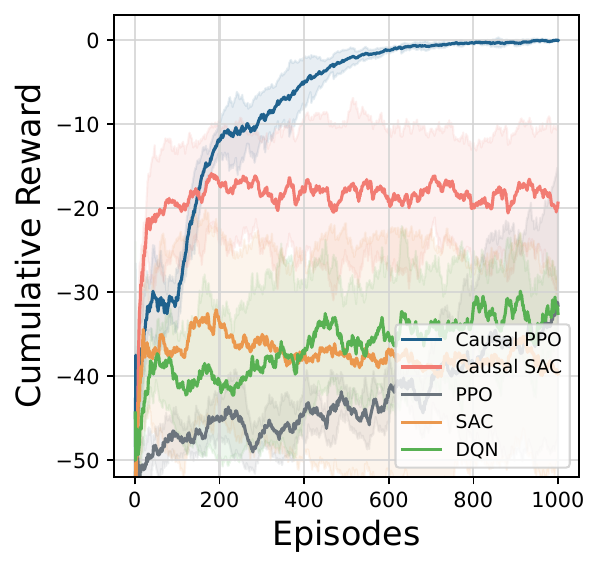}
    }
    \subfloat[] {
        \label{fig:policy_steps_notopo}
        \centering
        \includegraphics[width=0.31\linewidth]{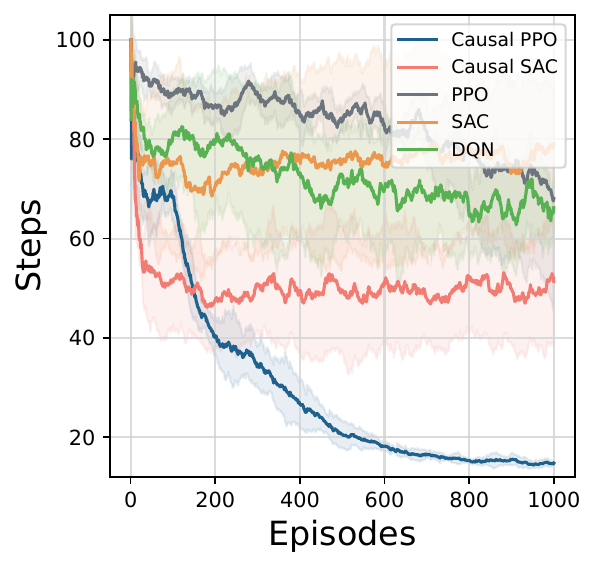}
    }
    \subfloat[] {
        \label{fig:policy_alarms_notopo}
        \centering
        \includegraphics[width=0.31\linewidth]{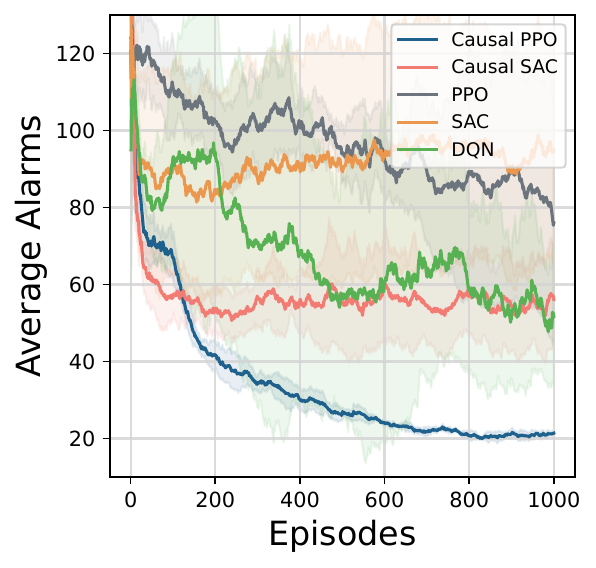}
    }
    \centering
    \caption{(a)-(c)Cumulative rewards, intervention steps, and average number of alarms per episode for Causal PPO based on random initialization structures at different $K$ in the topology-free environment.} 
    \label{fig:exp_policy_result_notopo}
\end{figure}
\begin{table*}
\liuhao
    \centering
    \caption{Results of causal structure learning of topology-free environment}
    
\begin{tabular}{l|c|c|c|c|c}
        \toprule
         Methods & F1 score & Precision & Recall & Accuracy & SHD\\
        \midrule
        Random Initiation & 0.006 $\pm$ 0.006 & 0.025 $\pm$ 0.025 & 0.003 $\pm$ 0.003 & 0.669 $\pm$ 0.983& 169.0 $\pm$ 5.362
        \\
        Causal PPO (Random) & \textbf{0.755} $\pm$ 0.023 & \textbf{0.814} $\pm$ 0.024 & \textbf{0.705} $\pm$ 0.025 & \textbf{0.993} $\pm$ 0.001 & \textbf{68.50} $\pm$ 6.225
        \\
        Causal SAC (Random) & 0.595 $\pm$ 0.027 & 0.558 $\pm$ 0.057 & 0.643 $\pm$ 0.017 & 0.987 $\pm$ 0.002 & 132.0 $\pm$ 15.859
        \\
        \bottomrule
    \end{tabular}
    \label{tab_causal_notopo}
\end{table*}

Considering that topology-free fault alarm scenarios also exist in real O\&M environments, we constructed another topology-free alarm environment with 100-dimensional alarm types based on real alarm data. The specific experimental configurations are shown in the Table~\ref{tab:env_configurations}. We also conducted comparative experiments in this environment. In policy learning, we used the model-free algorithms PPO~\cite{schulman2017proximal}, SAC~\cite{haarnoja2018soft}, and DQN~\cite{mnih2013playing} as baseline, and applied our method to PPO and SAC, resulting in Causal PPO and Causal SAC.
To better demonstrate the advantages of our method in causal structure learning, we use random graphs as the initial structures for the causal learning process.

As shown in Figure~\ref{fig:exp_policy_result_notopo}, our methods outperform the baseline algorithms in terms of cumulative rewards, number of interactions, and average number of alarms per episode metrics.
In terms of structure learning, discovering causality among 100-dimensional causal alarm nodes is challenging. 
However, as shown in Table~\ref{tab_causal_notopo}, compared to the randomized initial graph, our approaches can gradually learn a basic causal structure, which helps improve the convergence performance of the policy.
This also demonstrates the applicability of our algorithm in multiple scenarios.

\section{Additional experiment on cart-pole environment}\label{sec:carpole}

To evaluate the performance of our approach on classic control tasks, we included the \textit{cart-pole} environment from the OpenAI Gym toolkit. The cart-pole environment is a well-known benchmark in reinforcement learning, where the goal is to balance a pole on a moving cart by applying forces to the cart. The state space consists of the cart's position, velocity, pole angle, and pole angular velocity, while the action space is discrete, allowing the agent to push the cart either left or right.

In the cart-pole environment, there is a clear causal relationship between the pole’s angle and the cart’s acceleration: when the pole tilts to the right, continuing to apply force in that direction exacerbates the tilt, whereas applying force to the left helps restore balance. Leveraging this causal structure, we introduce a causal action masking mechanism that softly masks actions aligned with the tilt direction at extreme angles, thereby reducing ineffective exploration and expediting policy convergence. Specifically, since the goal $Y$ of cart-pole environment is to control the angle of pole, the causal mask is learned by setting it proportionally to the effect of the action $m_{\mathbf{s},i}^{\mathcal{G}} \propto |s_{\text{angle}}(I_i=1)|$ such that the action will more likely be masked if it increases the angle.

The experimental results (shown in Figure~\ref{fig:cartpole_reward}.) indicate that the proposed Causal PPO significantly outperforms other baselines in terms of cumulative rewards, and demonstrates faster convergence and higher stability during training, which fully proves that explicitly embedding causal inference in the action space is of key significance for efficient reinforcement learning of samples.

\section{Hyper-parameters}\label{sec:hyper-parameters}
We list all important hyper-parameters in the implementation for the FaultAlarmRL environment in Table~\ref{tab:algo_parameters}.
\begin{figure}[t]
    \centering
    \begin{minipage}[t]{0.48\linewidth}
        \vspace{0pt}
        \centering
        \includegraphics[width=\linewidth]{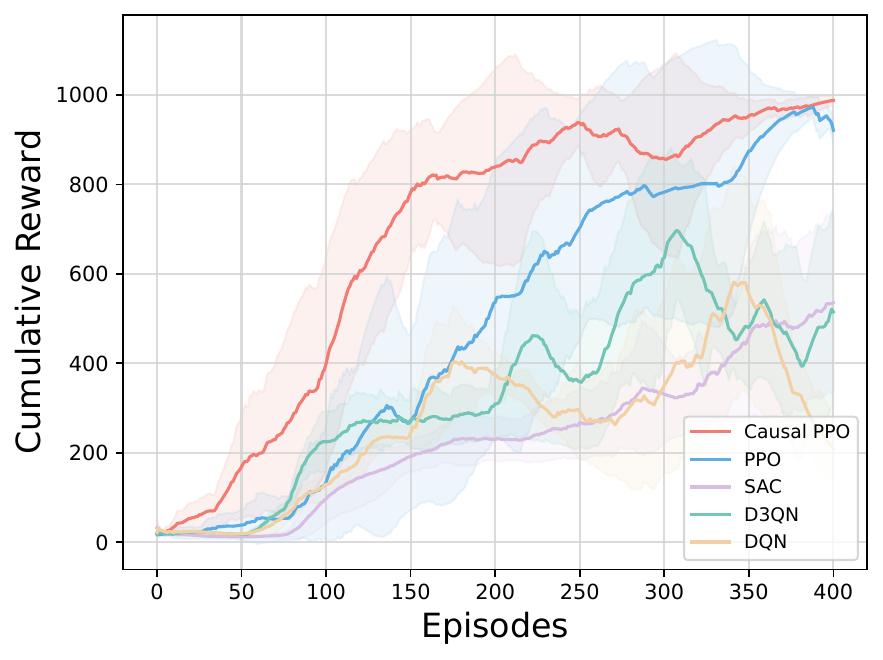}
        \caption{Cumulative rewards in the cart-pole environment.}
        \label{fig:cartpole_reward}
    \end{minipage}
    \hfill
    \begin{minipage}[t]{0.48\linewidth}
        \centering
     \vspace{0pt}
        \liuhao
        \renewcommand\arraystretch{1.1}
        \captionof{table}{Environment configurations used in experiments.}
        \label{tab:env_configurations}
        \begin{tabular}{l|c|c}
            \toprule
            \multirow{1}{*}{Environment} & \multirow{1}{*}{Parameters} & \multirow{1}{*}{Value} \\
            \midrule
            \multirow{10}{*}{\shortstack[c]{Topology\\environment}} 
            & Max step size          & 100 \\
            & State dimension        & 1800 \\
            & Action dimension       & 900  \\
            & Action type            & Discrete  \\
            & time range             & 50  \\
            & max hop                & 2 \\
            & $\alpha$ range         & [0.0001, 0.0013]  \\
            & $\mu$ range            & [0.0005, 0.0008]  \\
            & root cause num         & 50 \\
            \midrule
            \multirow{10}{*}{\shortstack[c]{Topology-free\\environment}} 
            & Max step size          & 100 \\
            & State dimension        & 200 \\
            & Action dimension       & 100  \\
            & Action type            & Discrete  \\
            & time range             & 100  \\
            & max hop                & 1 \\
            & $\alpha$ range         & [0.00015, 0.0025]  \\
            & $\mu$ range            & [0.0005, 0.0008]  \\
            & root cause num         & 20 \\
            \bottomrule
        \end{tabular}
    \end{minipage}
\end{figure}

\begin{table*}[t]
\liuhao
  \renewcommand\arraystretch{1.0}
  \centering
  \caption{Ground truth}
  \label{tab:ground_truth}
     \begin{tabular}{cc|cc}
    \toprule
    Cause & Effect & Cause & Effect \\
    \midrule
    MW\_RDI & LTI & MW\_BER\_SD & LTI \\
    MW\_RDI & CLK\_NO\_TRACE\_MODE & MW\_BER\_SD & S1\_SYN\_CHANGE \\
    MW\_RDI & S1\_SYN\_CHANGE & MW\_BER\_SD & PLA\_MEMBER\_DOWN \\
    MW\_RDI & LAG\_MEMBER\_DOWN & MW\_BER\_SD & MW\_RDI \\
    MW\_RDI & PLA\_MEMBER\_DOWN & MW\_BER\_SD & MW\_LOF \\
    MW\_RDI & ETH\_LOS & MW\_BER\_SD & ETH\_LINK\_DOWN \\
    MW\_RDI & ETH\_LINK\_DOWN & MW\_BER\_SD & NE\_COMMU\_BREAK \\
    MW\_RDI & NE\_COMMU\_BREAK & MW\_BER\_SD & R\_LOF \\
    MW\_RDI & R\_LOF & R\_LOF & LTI \\
    TU\_AIS & LTI   & R\_LOF & S1\_SYN\_CHANGE \\
    TU\_AIS & CLK\_NO\_TRACE\_MODE & R\_LOF & LAG\_MEMBER\_DOWN \\
    TU\_AIS & S1\_SYN\_CHANGE & R\_LOF & PLA\_MEMBER\_DOWN \\
    RADIO\_RSL\_LOW & LTI   & R\_LOF & ETH\_LINK\_DOWN \\
    RADIO\_RSL\_LOW & S1\_SYN\_CHANGE & R\_LOF & NE\_COMMU\_BREAK \\
    RADIO\_RSL\_LOW & LAG\_MEMBER\_DOWN & LTI   & CLK\_NO\_TRACE\_MODE \\
    RADIO\_RSL\_LOW & PLA\_MEMBER\_DOWN & HARD\_BAD & LTI \\
    RADIO\_RSL\_LOW & MW\_RDI & HARD\_BAD & CLK\_NO\_TRACE\_MODE \\
    RADIO\_RSL\_LOW & MW\_LOF & HARD\_BAD & S1\_SYN\_CHANGE \\
    RADIO\_RSL\_LOW & MW\_BER\_SD & HARD\_BAD & BD\_STATUS \\
    RADIO\_RSL\_LOW & ETH\_LINK\_DOWN & HARD\_BAD & POWER\_ALM \\
    RADIO\_RSL\_LOW & NE\_COMMU\_BREAK & HARD\_BAD & LAG\_MEMBER\_DOWN \\
    RADIO\_RSL\_LOW & R\_LOF & HARD\_BAD & PLA\_MEMBER\_DOWN \\
    BD\_STATUS & S1\_SYN\_CHANGE & HARD\_BAD & ETH\_LOS \\
    BD\_STATUS & LAG\_MEMBER\_DOWN & HARD\_BAD & MW\_RDI \\
    BD\_STATUS & PLA\_MEMBER\_DOWN & HARD\_BAD & MW\_LOF \\
    BD\_STATUS & ETH\_LOS & HARD\_BAD & ETH\_LINK\_DOWN \\
    BD\_STATUS & MW\_RDI & HARD\_BAD & NE\_COMMU\_BREAK \\
    BD\_STATUS & MW\_LOF & HARD\_BAD & R\_LOF \\
    BD\_STATUS & ETH\_LINK\_DOWN & HARD\_BAD & NE\_NOT\_LOGIN \\
    BD\_STATUS & RADIO\_RSL\_LOW & HARD\_BAD & RADIO\_RSL\_LOW \\
    BD\_STATUS & TU\_AIS & HARD\_BAD & TU\_AIS \\
    NE\_COMMU\_BREAK & LTI   & ETH\_LOS & LTI \\
    NE\_COMMU\_BREAK & CLK\_NO\_TRACE\_MODE & ETH\_LOS & CLK\_NO\_TRACE\_MODE \\
    NE\_COMMU\_BREAK & S1\_SYN\_CHANGE & ETH\_LOS & S1\_SYN\_CHANGE \\
    NE\_COMMU\_BREAK & LAG\_MEMBER\_DOWN & ETH\_LOS & LAG\_MEMBER\_DOWN \\
    NE\_COMMU\_BREAK & PLA\_MEMBER\_DOWN & ETH\_LOS & PLA\_MEMBER\_DOWN \\
    NE\_COMMU\_BREAK & ETH\_LOS & ETH\_LOS & ETH\_LINK\_DOWN \\
    NE\_COMMU\_BREAK & ETH\_LINK\_DOWN & MW\_LOF & LTI \\
    NE\_COMMU\_BREAK & NE\_NOT\_LOGIN & MW\_LOF & CLK\_NO\_TRACE\_MODE \\
    ETH\_LINK\_DOWN & LTI   & MW\_LOF & S1\_SYN\_CHANGE \\
    ETH\_LINK\_DOWN & CLK\_NO\_TRACE\_MODE & MW\_LOF & LAG\_MEMBER\_DOWN \\
    ETH\_LINK\_DOWN & S1\_SYN\_CHANGE & MW\_LOF & PLA\_MEMBER\_DOWN \\
    S1\_SYN\_CHANGE & LTI   & MW\_LOF & ETH\_LOS \\
    POWER\_ALM & BD\_STATUS & MW\_LOF & MW\_RDI \\
    POWER\_ALM & ETH\_LOS & MW\_LOF & ETH\_LINK\_DOWN \\
    POWER\_ALM & MW\_RDI & MW\_LOF & NE\_COMMU\_BREAK \\
    POWER\_ALM & MW\_LOF & MW\_LOF & R\_LOF \\
    \bottomrule
    \end{tabular}
\end{table*}

\begin{table*}[t]
\liuhao
\caption{Hyper-parameters of methods used in experiments.}
\label{tab:algo_parameters}
\centering
   \begin{tabular}{c|c|c}
    \toprule
    \multirow{1}{*}{Models} & \multirow{1}{*}{Parameters} & \multirow{1}{*}{Value} \\
    \midrule
    \multirow{10}{*}{Causal DQN \& Causal D3QN} 
    & Learning rate               & 0.0003  \\
    & Size of buffer $mathcal{B}$        & 100000 \\
    & Epoch per max iteration     & 100    \\
    & Batch size                  & 64     \\
    & Reward discount $\gamma$    & 0.99   \\
    & MLP hiddens                 & 128    \\
    & MLP layers                  & 2      \\
    & Update timestep             & 5      \\
    & Random sample timestep      & 512    \\
    & $\epsilon$-greedy ratio     & 0.1    \\
    & $\epsilon$-causal ratio $\eta$     & 0.2   \\
    \midrule
    \multirow{10}{*}{Causal PPO} 
    & Actor learning rate               & 0.0003  \\
    & Critic learning rate               & 0.0003  \\
    & Epoch per max iteration     & 100    \\
    & Batch size                  & 64     \\
    & Reward discount $\gamma$    & 0.99   \\
    & MLP hiddens                 & 128    \\
    & MLP layers                  & 2      \\
    & Clip                        & 0.2    \\
    & K epochs                    & 50     \\
    & Update timestep             & 256    \\
    & Random sample timestep      & 512    \\
    & $\epsilon$-greedy ratio     & 0.1    \\
    & $\epsilon$-causal ratio $\eta$     & 0.3   \\
    \midrule
    \multirow{10}{*}{DQN \& D3QN} 
    & Learning rate               & 0.0003  \\
    & Size of buffer $mathcal{B}$        & 100000 \\
    & Epoch per max iteration     & 100    \\
    & Batch size                  & 64     \\
    & Reward discount $\gamma$    & 0.99   \\
    & MLP hiddens                 & 128    \\
    & MLP layers                  & 2      \\
    & Update timestep             & 5      \\
    & Random sample timestep      & 512    \\
    & $\epsilon$-greedy ratio     & 0.1    \\
    \midrule
    \multirow{10}{*}{PPO} 
    & Actor learning rate               & 0.0003  \\
    & Critic learning rate               & 0.0003  \\
    & Epoch per max iteration     & 100    \\
    & Batch size                  & 64     \\
    & Reward discount $\gamma$    & 0.99   \\
    & MLP hiddens                 & 128    \\
    & MLP layers                  & 2      \\
    & Clip                        & 0.2    \\
    & K epochs                    & 50     \\
    & Update timestep             & 512    \\
    & Random sample timestep      & 512    \\
    \bottomrule
  \end{tabular}
\end{table*}

\end{appendix}

\end{document}